\pdfoutput=1
\documentclass{article}

\usepackage[main, final]{neurips_2026}

\usepackage[utf8]{inputenc}
\usepackage[T1]{fontenc}
\usepackage{hyperref}
\usepackage{url}
\usepackage{booktabs}
\usepackage{array}
\usepackage{amsfonts}
\usepackage{amsmath}
\usepackage{amssymb}
\usepackage{amsthm}
\usepackage{nicefrac}
\usepackage{microtype}
\usepackage{xcolor}
\usepackage{graphicx}
\usepackage{algorithm}
\usepackage{algorithmic}
\usepackage{enumitem}
\usepackage{listings}
\usepackage{float}

\newtheorem{theorem}{Theorem}
\newtheorem{proposition}[theorem]{Proposition}
\newtheorem{lemma}[theorem]{Lemma}

\theoremstyle{definition}
\newtheorem{definition}[theorem]{Definition}

\theoremstyle{remark}
\newtheorem{remark}[theorem]{Remark}

\DeclareMathOperator{\im}{Im}
\newcommand{\R}{\mathbb{R}}
\newcommand{\C}{\mathbb{C}}
\newcommand{\SO}{\mathrm{SO}}

\newcommand{\SH}{Y}
\newcommand{\CG}[6]{C^{#5,\,#6}_{#1,\,#2;\,#3,\,#4}}
\newcommand{\abs}[1]{\lvert #1 \rvert}

\newcommand{\conj}[1]{\overline{#1}}

\theoremstyle{definition}

\setlist[itemize]{noitemsep,topsep=2pt,partopsep=0pt}
\setlist[enumerate]{noitemsep,topsep=2pt,partopsep=0pt}
\hypersetup{pdftitle={bispectrum: Selective G-Bispectra Made Practical}}
\title{\texorpdfstring{%
  \mbox{\vphantom{\texttt{bispectrum}}%
    \raisebox{-0.165em}{\includegraphics[width=4.92em]{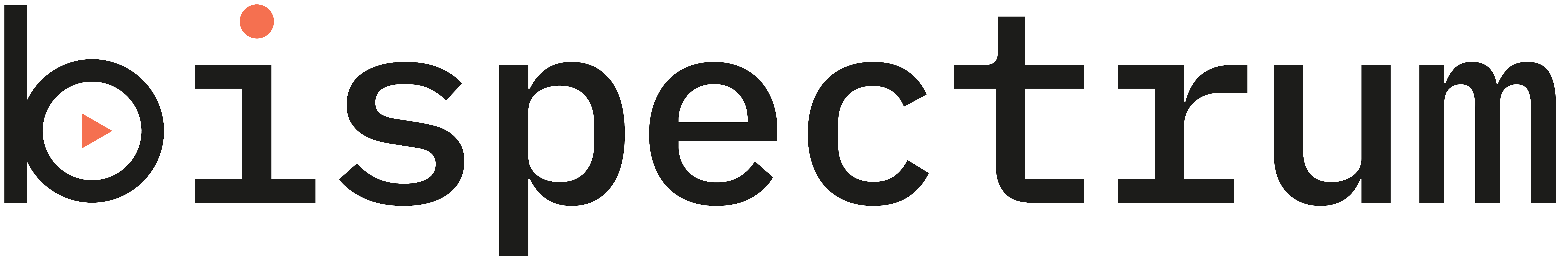}}}%
}{bispectrum}: Selective $G$-Bispectra Made Practical}

\definecolor{bspLinkOrange}{HTML}{CC4A28}
\definecolor{bspInk}{HTML}{1C1C1A}
\hypersetup{
  colorlinks=true,
  linkcolor=bspLinkOrange,
  citecolor=bspInk,
  urlcolor=bspLinkOrange,
}
\newcommand{\repolink}{%
  \href{https://github.com/geometric-intelligence/bispectrum}{%
    \raisebox{-0.2em}{\includegraphics[height=1.15em]{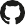}}\hspace{0.45em}%
    \textcolor{bspLinkOrange}{\texttt{geometric-intelligence/bispectrum}}}%
}

\newcommand{\authorblock}[4]{%
  \begin{minipage}[t]{0.43\textwidth}
    \vspace{0pt}
    \centering
    \textbf{#1}\\
    #2\\
    #3\\
    \texttt{#4}
  \end{minipage}%
}

\author{%
\normalfont
\begin{tabular*}{\textwidth}{@{\extracolsep{\fill}}cc@{}}
  \authorblock
    {Johan Mathe}
    {Atmo, Inc.}
    {San Francisco, CA}
    {johan@atmo.ai}
  &
  \authorblock
    {Adele Myers Lantow}
    {UC Santa Barbara}
    {Santa Barbara, CA}
    {adele@ucsb.edu}
  \\[5em]
  \authorblock
    {Simon Mataigne}
    {UCLouvain}
    {Louvain-la-Neuve, Belgium}
    {simon.mataigne@uclouvain.be}
  &
  \authorblock
    {Nina Miolane}
    {UC Santa Barbara}
    {Santa Barbara, CA}
    {ninamiolane@ucsb.edu}
\end{tabular*}%
\\\noalign{\vskip 1.1em}
\normalfont\repolink
}

\begin{document}

\maketitle
\begin{figure}[H]
  \centering
  \begingroup
  \definecolor{atlasInk}{HTML}{1C1C1A}
  \color{atlasInk}
  \setlength{\tabcolsep}{0pt}
  \newcommand{\atlasname}[1]{{\sffamily\small #1}}
  \newcommand{\atlasaction}[2]{{\footnotesize $#1\curvearrowright #2$}}
  \newcommand{\atlasfamily}[1]{{\color{atlasInk}\sffamily\bfseries\footnotesize #1}}
  \begin{tabular*}{\linewidth}{@{\extracolsep{\fill}}ccccccc@{}}
    \multicolumn{3}{c}{\atlasfamily{Finite groups}} &
      \multicolumn{4}{c}{\atlasfamily{Continuous domains}}\\[-1pt]
    \cmidrule(lr){1-3}\cmidrule(lr){4-7}
    \includegraphics[width=29pt]{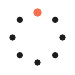} &
    \includegraphics[width=29pt]{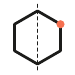} &
    \includegraphics[width=29pt]{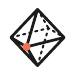} &
    \includegraphics[width=29pt]{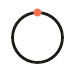} &
    \includegraphics[width=29pt]{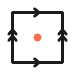} &
    \includegraphics[width=29pt]{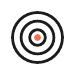} &
    \includegraphics[width=29pt]{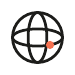}\\[-1pt]
    \atlasname{Cyclic} & \atlasname{Dihedral} & \atlasname{Octahedral} &
    \atlasname{Circle} & \atlasname{Torus} & \atlasname{Disk} & \atlasname{Sphere}\\[1pt]
    \atlasaction{C_n}{C_n} & \atlasaction{D_n}{D_n} & \atlasaction{O}{O} &
    \atlasaction{\mathrm{SO}(2)}{S^1} & \atlasaction{\mathbb{T}^d}{\mathbb{T}^d} &
    \atlasaction{\mathrm{SO}(2)}{\mathbb{D}} & \atlasaction{\mathrm{SO}(3)}{S^2}
  \end{tabular*}
  \endgroup
  \caption{The seven group actions $G\curvearrowright X$ (group $G$ acting on domain $X$) in \texttt{bispectrum}.}
  \label{fig:front-group-actions}
\end{figure}
\begin{abstract}
Many machine learning tasks are invariant under the action of a group $G$ of transformations. Signal classification
can be invariant under translations, image
classification under 2D rotations, and spherical-image classification under 3D rotations.
One way to build this invariance into a neural network is through $G$-invariants. The $G$-bispectrum is a complete one. It determines the signal up to the group action and has served as a pooling layer in deep networks. Yet it is costly to compute, and each existing implementation covers a single group, severely limiting its adoption.
We present \texttt{bispectrum}, an open-source, fully unit-tested
PyTorch library that implements computationally efficient \emph{selective} $G$-bispectra for seven group actions as differentiable modules. Selectivity
reduces the number of bispectral coefficients from $O(|G|^2)$ to $O(|G|)$ for finite groups $G$, from $O(m^{5/2})$ to $O(m)$ for planar rotations on the disk with $m$ harmonics, and, with the augmented selective spherical bispectrum we introduce, from $\Theta(L^3)$ to $\Theta(L^2)$ for spherical 3D rotations at band limit $L$.
The library is invariant up to floating-point error and computes selective $G$-bispectra in sub-millisecond time on GPU. We evaluate them as pooling layers in $G$-equivariant networks on PatchCamelyon ($C_8$) and OrganMNIST3D (the octahedral group $O$) against five baselines. Across 46 comparisons, bispectral pooling remains competitive, matching all baselines on PCam and improving on max pooling by 5.0\,pp on OrganMNIST3D.
On Spherical MNIST, the augmented selective spherical bispectrum with an MLP head reaches 95.8\% on rotated test digits, on par with a matched spherical CNN at $4\times$ lower latency and within 0.3\,pp of a tensor-product network at $364\times$ lower latency.
Selective $G$-bispectra are thus fast and competitive invariant layers that keep far more of the signal than standard pooling, for both finite and continuous groups, and we hope \texttt{bispectrum} will broaden their use to new applications.
\end{abstract}
\section{Introduction}
\label{sec:intro}

In many learning problems the label does not change when a group $G$ of transformations acts on
the input. A translated signal, a rotated image, or a rotated spherical image keeps its class.
Equivariant neural networks~\citep{cohen2016group, weiler2019general} enforce invariance by adding an invariant pooling layer to their equivariant feature maps.
Existing choices (norm pooling, gate-based pooling, max/sum
aggregation~\citep{weiler2019general}) are simple but
\emph{incomplete}. They discard information that a complete
group-invariant representation would preserve. For instance, max pooling over $G$ keeps only the largest response across the group, so two feature maps that share their maximum but differ everywhere else are mapped to the same invariant
descriptor even though no group element relates them.

The $G$-bispectrum~\citep{tukey1953spectral, giannakis1989bispectrum, kakarala2012bispectrum,
sanborn2023bispectral, mataigne2024selective, myers2025selective}, a complex-valued
$G$-invariant descriptor, avoids this. It is \textit{complete} (it preserves all signal
information up to the group action), \textit{invertible} (the signal can be recovered from
it up to a group action), and \textit{selective} (a small subset of its entries suffices for
completeness on generic signals). Selectivity was first observed for translations
by~\citet{giannakis1989bispectrum} and recently extended to finite groups and the
disk~\citep{mataigne2024selective,myers2025selective}.

\paragraph{Prior implementations.} Existing implementations and uses of bispectra are
fragmented across works, each tied to one group and its own code. The selective $G$-bispectrum
of~\citet{mataigne2024selective} covers signals defined over finite groups. The selective disk bispectrum of
\citet{myers2025selective} covers 2D images defined over the disk, with invariance with respect to planar rotations. The spherical bispectrum of \citet{kakarala2012bispectrum} provides a representation of spherical images that is invariant to 3D rotations. \citet{oreiller2020solid} use it as a locally rotation-invariant layer in 3D CNNs, keeping all $\Theta(L^3)$ coefficients up to a maximal degree~$L$, and \citet{chevalley2024bispectral} build a bispectral layer into a 3D U-Net for segmentation. The classical translation bispectrum~\citep{tukey1953spectral, giannakis1989bispectrum} is well-known in signal processing but lacks an open, GPU-ready implementation. \citet{sanborn2023bispectral} learn the bispectrum of compact commutative groups from data with bispectral neural networks, a form that does not cover non-commutative groups.
No single library offers these invariants across groups behind one tested, GPU-ready
interface, so using a bispectrum on a new group or domain still means re-deriving and
re-implementing it. \texttt{bispectrum} closes this gap.

\paragraph{Contributions.} We present \texttt{bispectrum}, an open-source PyTorch library of
\emph{selective} $G$-bispectra. Our contributions are:
\begin{enumerate}[leftmargin=*]
  \item \textbf{Library.} We implement selective
    $G$-bispectra for \emph{seven group actions} as differentiable modules. We provide an efficient implementation with sub-millisecond GPU forward passes and a fully unit-tested codebase (98.5\% line coverage).
  \item \textbf{Spherical invariant.}
    We introduce the augmented selective spherical bispectrum for 3D rotations acting on spherical images. At band limit $L$, it combines a selective subset of the bispectrum with a few Clebsch--Gordan power coefficients, reducing the number of coefficients from $\Theta(L^3)$ to $\Theta(L^2)$. Its per-degree Jacobian has full rank by construction, and a degree-by-degree solver recovers every test signal from it at $L{=}15$. We conjecture that it is generically complete.
 \item \textbf{Evaluation.} We evaluate bispectral
  pooling in $G$-equivariant networks on 2D histopathology ($C_8$) and 3D organ
  ($O$) classification, against norm pooling, gated pooling, Fourier-ELU
  pooling, max pooling, and a data-augmented CNN. Across 46 comparisons it
  remains competitive, matching all baselines on PCam and
  improving on max pooling by 5.0\,pp on OrganMNIST3D. On Spherical MNIST, the
  augmented selective spherical bispectrum with an MLP head is within 0.3\,pp of matched learned-equivariant
  networks on rotated digits at $4\times$ to $364\times$ lower latency. We also find that a
  ReLU before the first learned projection costs 4--5\,pp on OrganMNIST3D.
\end{enumerate}
Overall, these results position selective $G$-bispectra as a practical, information-preserving
alternative to the incomplete pooling layers of equivariant networks.

\section{Mathematical Background}
\label{sec:background}

A group $G$ \emph{acts} on a domain $X$, for example rotations acting on the sphere, through a
map $(g,x)\mapsto g\cdot x$. This action carries over to signals defined on that domain,
$(f\circ g)(x):= f(g\cdot x)$. Two signals are in the same \emph{orbit}
if and only if they are related by some $g\in G$, like an image and a rotated copy of it. An
\emph{invariant} is a function that is constant on each orbit: it takes the same value on a
signal and on all of its transformed copies.

Equivariant networks process signals while keeping track of the group action.
A neural network layer $\Phi$ is \emph{$G$-equivariant} if
$\Phi(\pi_{\mathrm{in}}(g)\, f) = \pi_{\mathrm{out}}(g)\, \Phi(f)$
for all $g \in G$~\citep{cohen2016group, weiler2019general}, where $\pi_{\mathrm{in}}(g)$ and
$\pi_{\mathrm{out}}(g)$ denote how $g$ acts on the layer's input and output feature maps.
Transforming the input thus transforms the output in a predictable way.
Stacking equivariant layers preserves equivariance, but any task requiring
a prediction independent of the group action must eventually map equivariant
features to invariant ones. This is the role of the invariant pooling layer, and the
$G$-bispectrum is one choice for it.

The $G$-bispectrum is built from the Fourier transform on $G$, which relies on representations.
A \emph{representation} $\rho$ of $G$ assigns to each $g\in G$ a unitary matrix $\rho(g)$ in a
way that respects composition, $\rho(gh)=\rho(g)\rho(h)$. Representations generalize the
complex exponentials $e^{in\phi}$ of classical Fourier analysis, which represent planar
rotations by $1\times1$ matrices. A representation is \emph{irreducible} (an \emph{irrep}) when
it cannot be split into smaller ones, that is, when it has no nontrivial invariant subspace.
Irreps play the role of frequencies, and the \emph{dual} $\hat{G}$ is the set of equivalence
classes of irreps.

\paragraph{Bispectrum on finite groups.}
The $G$-bispectrum~\citep{kakarala1992triple, mataigne2024selective}
provides a complete invariant for signals on finite
groups, usable as a pooling layer. Given a signal $f : G \to \R$ on a finite group $G$ with
unitary irreps $\{\rho_i\}_{i\in\hat{G}}$, the \emph{$G$-Fourier
transform}~\citep{diaconis1990fourier} is
\begin{equation}\label{eq:fourier-finite}
  \mathcal{F}(f)_{\rho_i}
  := \sum_{g \in G} f(g)\, \rho_i(g)^{\dagger}
  \;\in\; \C^{d_i \times d_i},
\end{equation}
where $d_i = \dim(\rho_i)$ and $\dagger$ denotes the conjugate transpose. The \emph{full
$G$-bispectrum} is
\begin{equation}\label{eq:bispectrum-def}
  \beta(f)_{\rho_1,\rho_2}
  = \bigl[\mathcal{F}(f)_{\rho_1} \otimes
    \mathcal{F}(f)_{\rho_2}\bigr]\,
    C_{\rho_1,\rho_2}\,
    \Bigl[\bigoplus_{\rho \in \rho_1 \otimes \rho_2}
      \mathcal{F}(f)_{\rho}^{\dagger}\Bigr]\,
    C_{\rho_1,\rho_2}^{\dagger},
\end{equation}
where $C_{\rho_1,\rho_2}$ is the unitary Clebsch--Gordan matrix that
block-diagonalizes the tensor product $\rho_1 \otimes \rho_2$ into
irreducible summands, and $\bigoplus$ forms the block-diagonal matrix with one block
$\mathcal{F}(f)_\rho^\dagger$ for each irrep $\rho$ in the decomposition of
$\rho_1\otimes\rho_2$, repeated by multiplicity. Each matrix
$\beta(f)_{\rho_1,\rho_2}$ is $G$-invariant, and the full collection
$\{\beta(f)_{\rho_1,\rho_2}\}_{(\rho_1,\rho_2)\in\hat{G}^2}$
is a complete $G$-invariant for generic signals (i.e., those with
nonsingular Fourier coefficients)~\citep{kakarala1992triple}.
The full bispectrum has $O(|G|^2)$ scalar entries.
\citet{mataigne2024selective} show that a breadth-first search over
the Kronecker product table of $\hat{G}$, which lists the irreps in each
$\rho_i\otimes\rho_j$, selects $O(|G|)$ bispectral
coefficients that remain complete. For example, the octahedral group
drops from $576$ to $172$ scalars.

\paragraph{Bispectrum on the disk.}
The disk bispectrum~\citep{zhao2014rotationally, myers2025selective}
is an $\SO(2)$-invariant descriptor for planar images defined on the
unit disk $\mathbb{D}$. A signal $f : \mathbb{D} \to \R$ is expanded in
disk harmonics (eigenfunctions of the Laplacian on $\mathbb{D}$):
\begin{equation}\label{eq:disk-harmonics}
  \psi_{nk}(r,\theta) = c_{nk}\, J_n(\lambda_{nk}\, r)\, e^{in\theta},
\end{equation}
where $J_n$ is the Bessel function of order $n$, $\lambda_{nk}$ its
$k$-th positive root, and $c_{nk}$ a normalization constant.
The disk harmonic (DH) coefficients $a_{n,k}$ are equivariant to planar
rotations: if $f' = f \circ R_\phi$ for the rotation $R_\phi$ by angle $\phi$, then
$a^{f'}_{n,k} = e^{in\phi}\, a^{f}_{n,k}$.
For a band limit yielding $m$ coefficients with maximum angular frequency
$N_m$, the \emph{full disk bispectrum}~\citep{zhao2014rotationally} is
\begin{equation}\label{eq:bispectrum-disk}
  b_{j_1,j_2,k_3}
  = a_{n_{j_1},k_{j_1}}\, a_{n_{j_2},k_{j_2}}\,
    a_{n_{j_1}+n_{j_2},\,k_3}^{*},
\end{equation}
where $j\mapsto(n_j,k_j)$ enumerates the $m$ coefficients. It has $O(m^3 / N_m)$
coefficients, that is $O(m^{5/2})$, since a band limit on the Bessel
roots gives $N_m = \Theta(\sqrt{m})$. \citet{myers2025selective} introduce
the \emph{selective disk bispectrum}, which keeps only the products
$a_{0,1}^2\, a_{0,k}^{*}$ and $a_{1,1}\, a_{n,1}\, a_{n+1,k}^{*}$, one per coefficient and
so $O(m)$ in total, while remaining a complete
$\SO(2)$-invariant (under the genericity condition $a_{n,1}\neq 0$ for
all angular orders $n$).

\paragraph{Bispectrum on the sphere.}
The spherical bispectrum~\citep{kakarala2012bispectrum} is an
$\SO(3)$-invariant descriptor for band-limited signals on $S^2$.
A real signal $f : S^2 \to \R$ at band limit $L$ has spherical harmonic
coefficients $a_\ell^m$ with coefficient vectors
$\mathbf{F}_\ell = (a_\ell^{-\ell},\ldots,a_\ell^{\ell})^\top
\in \C^{2\ell+1}$, transforming under $g\in\SO(3)$ as
$\mathbf{F}_\ell \mapsto D^\ell(g)\,\mathbf{F}_\ell$, where $D^\ell$
is the $(2\ell{+}1)$-dimensional Wigner $D$-matrix.
For each admissible triple $(\ell_1,\ell_2,\ell)$ satisfying the
triangle inequality
$|\ell_1 - \ell_2| \le \ell \le \ell_1 + \ell_2$, the
\emph{bispectral coefficient} is
\begin{equation}\label{eq:bispectrum-sphere}
  \beta_{\ell_1,\ell_2,\ell}
  = \sum_{\substack{m_1,\,m_2,\,m \\ m_1+m_2=m}}
    \CG{\ell_1}{m_1}{\ell_2}{m_2}{\ell}{m}\;
    a_{\ell_1}^{m_1}\,a_{\ell_2}^{m_2}\,\conj{a_\ell^m},
\end{equation}
where $\CG{\ell_1}{m_1}{\ell_2}{m_2}{\ell}{m}$ denotes the
Clebsch--Gordan coefficient
$\langle \ell_1, m_1;\, \ell_2, m_2 \mid \ell, m \rangle$.
The full set has $\Theta(L^3)$ scalars. On homogeneous spaces such as
$S^2$, where $G$ can move any point to any other, the Fourier coefficients are rank-deficient
(vectors rather than full-rank matrices), so the completeness theorem
of~\citet{kakarala2012bispectrum}, which requires nonsingular matrix coefficients, does not
directly apply~\citep{kondor2007novel}. In Section~\ref{sec:so3-selective} we
introduce a selective construction that reduces the scalar bispectrum to
$\Theta(L^2)$ entries and test its completeness empirically by
signal reconstruction.

Strictly, the $G$-bispectrum~\citep{kakarala1992triple, mataigne2024selective} is defined for
signals on the group $G$ itself (Equation~\eqref{eq:bispectrum-def}), whereas the disk and
sphere bispectra (Equations~\eqref{eq:bispectrum-disk} and~\eqref{eq:bispectrum-sphere}) are
bispectra of signals on a domain acted on by $G$. For brevity, we extend the term
\emph{$G$-bispectrum} to all three.

\section{Related Work}
\label{sec:related}

\paragraph{Equivariant network libraries.}
The \texttt{escnn}~\citep{cesa2022program} and \texttt{e3nn}~\citep{geiger2022e3nn}
libraries provide general-purpose equivariant convolution layers for 2D and 3D
groups respectively, and \texttt{lie-nn}~\citep{batatia2023general}
provides general Lie-group representation primitives (irreducible
representations, Clebsch--Gordan decompositions, tensor products) on
top of which custom equivariant layers can be built. These libraries
implement equivariant \emph{feature maps} but leave the choice of
invariant pooling to the user (typically norm or gate). Our modules provide the invariant map and work with any of these libraries.

\paragraph{Bispectral approaches in neural networks.}
\citet{sanborn2023bispectral} introduced bispectral neural networks (BNNs),
which learn a Fourier transform and its full bispectrum from data for compact
commutative groups (2D cyclic translations, planar rotations).
\citet{oreiller2022bispectral} build a locally rotation-invariant layer for a
2D U-Net from the bispectrum of local circular-harmonic responses, over all
frequency pairs up to a maximal degree. \citet{oreiller2020solid} do the same
in 3D with spherical harmonics, and \citet{chevalley2024bispectral} build such
a layer into a 3D U-Net for segmentation. Both Oreiller layers keep every
bispectral coefficient up to symmetry, and each of these implementations is
tied to a specific architecture. Our library
covers seven group/domain pairs, defaults to the
$O(|G|)$ selective reduction~\citep{mataigne2024selective}, and exposes
each invariant as a standalone PyTorch module.
Outside the equivariant-network setting,
\texttt{PyBispectra}~\citep{binns2025pybispectra} computes bispectrum
and bicoherence statistics of one-dimensional time-series signals for
electrophysiology applications such as phase-amplitude coupling and
time-delay estimation. The underlying 1D bispectrum is translation-invariant
by construction, but the toolbox is not designed as a neural-network
pooling layer.

\paragraph{Bispectra on homogeneous and non-homogeneous spaces.}
The selective $G$-bispectrum
construction~\citep{mataigne2024selective} reduces the cost of the
$G$-bispectrum from $O(|G|^2)$ to $O(|G|)$ scalars on finite groups.
\citet{myers2025selective} extend the selective construction to the
disk bispectrum of~\citet{zhao2014rotationally}, where $\SO(2)$ acts on a non-homogeneous space,
since rotations preserve the radius.
\citet{kakarala2012bispectrum} treats the $\SO(3)$-on-$\SO(3)$
bispectrum and the closely related sphere case via Wigner-$D$ irreps.
The matrix (vector) bispectrum on $S^2$ is the homogeneous-space form
on which our scalar contractions and CG power augmentation are built.

\paragraph{Scattering transforms.}
The wavelet scattering transform~\citep{mallat2012scattering} provides
translation-invariant representations via iterated modulus operations.
\citet{mcewen2022scattering} extend it to the sphere. Unlike the
$G$-bispectrum, scattering representations are stable but generally
incomplete. They discard inter-scale phase information.

\begin{table}[t]
  \caption{\textbf{Selective bispectra are complete invariants with as few coefficients as
    incomplete ones.} The selective $G$-bispectrum~\citep{mataigne2024selective} and the
    selective disk bispectrum~\citep{myers2025selective} are complete with $O(|G|)$ and $O(m)$
    coefficients, for $m$ disk harmonics. Our augmented selective spherical bispectrum has $\Theta(L^2)$
    coefficients and is complete in all our tests (Sec.~\ref{sec:so3-selective}), without a
    proof. $^\dagger$Completeness holds for almost every signal, all but a negligible set.}
  \label{tab:nonlinearity-comparison}
  \centering
  \small
  \setlength{\tabcolsep}{4pt}
  \begin{tabular}{lcccl}
    \toprule
    Nonlinearity & Complete? & Coefficients & Forward cost & Key limitation \\
    \midrule
    Norm              & No              & $O(|G|)$   & $O(|G|)$          & Discards phase \\
    Gate              & No              & $O(|G|)$   & $O(|G|)$          & Scalar bottleneck \\
    Fourier pointwise & No$^\ddagger$   & $O(|G|)$   & $O(|G| \log |G|)$ & Aliasing \\
    Full bispectrum (finite $G$)      & Yes & $O(|G|^2)$   & $O(|G|^3)$ & Quadratic cost \\
    Full bispectrum ($\SO(2)$/disk)   & Yes & $O(m^{5/2})$  & $O(m^{5/2})$ & Superlinear cost \\
    Full bispectrum ($\SO(3)$/$S^2$)  & Yes & $\Theta(L^3)$ & $O(L^5)$  & Cubic cost \\
    \midrule
    Selective bispectrum (finite $G$) & Yes       & $O(|G|)$      & $O(|G|^2)$ & Genericity assumption$^\dagger$ \\
    Selective bispectrum ($\SO(2)$/disk) & Yes    & $O(m)$        & $O(m \log m)^\S$ & Genericity assumption$^\dagger$ \\
    \begin{tabular}[t]{@{}l@{}}Augmented selective spherical\\ \quad bispectrum ($\SO(3)$/$S^2$, ours)\end{tabular} & Empirical & $\Theta(L^2)$ & $O(L^4)$   & No completeness proof \\
    \bottomrule
  \end{tabular}

  \vspace{0.3em}
  {\small $^\ddagger$The pointwise nonlinearity is equivariant, not invariant, so an invariant still needs pooling over $G$ (max pooling in our experiments), which is incomplete. Aliasing also breaks exact equivariance. $^\S$With the fast disk-harmonic transform of \citet{marshall2023fast}. Our implementation uses an exact dense transform, $O(m^2)$.}
\end{table}

Of the layers compared in Table~\ref{tab:nonlinearity-comparison}, only the selective
bispectrum (Section~\ref{sec:background}) is complete on generic signals with
$O(|G|)$ coefficients.

\section{The \texttt{bispectrum} Library}
\label{sec:library}

\paragraph{Design and implementation.} \label{sec:design} Every $G$/domain pair is a \texttt{torch.nn.Module} named
\texttt{\{Group\}on\{Domain\}} (e.g., \texttt{SO3onS2}).
Table~\ref{tab:supported-groups} lists all seven modules. Each exposes a
uniform API: \texttt{forward(f)} returns the selective invariants,
\texttt{fourier(f)} the group Fourier coefficients, and
\texttt{invert(beta)} the reconstructed signal up to group-action
indeterminacy (where available). Modules default to $O(|G|)$
selective coefficients (\texttt{selective=False} restores the full
$O(|G|^2)$ set where available, Table~\ref{tab:supported-groups}), precompute CG matrices / DFT kernels / Bessel
roots as non-learnable buffers, and depend only on
PyTorch~\citep{paszke2019pytorch}, NumPy~\citep{harris2020numpy},
and \texttt{torch\_harmonics}~\citep{bonev2023spherical} (for
\texttt{SO3onS2}). Test suite line coverage is 98.5\%. A minimal
usage example is in Appendix~\ref{app:api-example}.

\begin{table}[tbp]
  \caption{Seven group/domain pairs in the \texttt{bispectrum} library.
    \texttt{SO2onS1} is the continuous-$n$ limit of \texttt{CnonCn}.}
  \label{tab:supported-groups}
  \centering
  \small
  \begin{tabular*}{\linewidth}{@{\extracolsep{\fill}}llll@{}}
    \toprule
    Module & Group / Domain & Output mode & Reference \\
    \midrule
    \texttt{CnonCn} & $C_n$ on $C_n$ & selective + full & \citep{giannakis1989bispectrum} \\
    \texttt{SO2onS1} & $\SO(2)$ on $S^1$ & selective + full & \citep{mataigne2024selective} \\
    \texttt{TorusOnTorus} & $\mathbb{T}^d$ on $\mathbb{T}^d$ & selective + full & \citep{mataigne2024selective} \\
    \texttt{DnonDn} & $D_n$ on $D_n$ & selective & \citep{mataigne2024selective} \\
    \texttt{SO2onDisk} & $\SO(2)$ on disk & selective & \citep{myers2025selective} \\
    \texttt{SO3onS2} & $\SO(3)$ on $S^2$ & selective + full & \citep{kakarala2012bispectrum}, this paper \\
    \texttt{OctaonOcta} & chiral octahedral $O$ & selective & \citep{mataigne2024selective} \\
    \bottomrule
  \end{tabular*}
\end{table}

\paragraph{Profiling.} \label{sec:benchmarks} We profile the library on a single NVIDIA H100 80\,GB GPU
(PyTorch 2.11, \texttt{bispectrum} 0.4.1, batch\,=\,16 for the forward pass and 4 for inversion,
median wall-clock via \texttt{torch.utils.benchmark}). Results are in
Table~\ref{tab:benchmarks}.
With autograd enabled, as in training, the selective forward pass is sub-millisecond for all six modules,
from 0.08\,ms ($\mathbb{T}^2$) to 0.68\,ms (octahedral group). \texttt{SO3onS2} at $L{=}16$
takes 0.63\,ms, or 0.20\,ms in inference mode, where it replays a CUDA graph.
GPU throughput exceeds $10^7$~samples/s for $C_n$ and $\mathbb{T}^2$
(Appendix~\ref{app:bench}). The same configurations stay sub-millisecond on an A100, a consumer RTX 4090,
and Apple silicon (MPS). On CPU, all but $\mathrm{SO}(3)$ on $S^2$ do
(Appendix~\ref{app:portability}). Six modules are presented. The seventh,
\texttt{SO2onS1}, subclasses \texttt{CnonCn} with $n = L{+}1$ and has
identical runtime to the corresponding \texttt{CnonCn} row at matched
$|G|$. Coefficient-count scaling and additional plots are in
Appendix~\ref{app:bench} (Figure~\ref{fig:coeff-scaling}).

\begin{table}[tbp]
  \centering
  \caption{\textbf{The selective forward pass is sub-millisecond for every module, with up to
    $512\times$ fewer coefficients than the full bispectrum.}
    Median wall-clock, NVIDIA H100, PyTorch 2.11, batch\,=\,16, with autograd
    enabled (Sel.) and in inference mode (Inf.), where \texttt{SO3onS2} replays a
    CUDA graph. ``--'': not implemented.}
  \label{tab:benchmarks}
  \small
  \setlength{\tabcolsep}{3.5pt}
  \begin{tabular*}{\linewidth}{@{\extracolsep{\fill}}llrrrrrr@{}}
    \toprule
    Module & $G$ & $|G|/L$ & \multicolumn{2}{c}{Coefs} & \multicolumn{3}{c}{Fwd (ms)} \\
    \cmidrule(lr){4-5} \cmidrule(lr){6-8}
     & & & Sel. & Full & Sel. & Inf. & Full \\
    \midrule
    \texttt{CnonCn} & $C_{128}$ & 128 & 128 & 8{,}256 & 0.14 & 0.13 & 8.81 \\
    \texttt{TorusOnTorus} & $C_{32}{\times}C_{32}$ & 1{,}024 & 1{,}024 & 524{,}800 & 0.08 & 0.07 & 0.31 \\
    \texttt{DnonDn} & $D_{32}$ & 64 & 245 & -- & 0.54 & 0.48 & -- \\
    \texttt{SO2onDisk} & $\mathrm{SO}(2)$ & $L{=}16$ & 76 & -- & 0.22 & 0.21 & -- \\
    \texttt{SO3onS2} & $\mathrm{SO}(3)$ & $L{=}16$ & 426 & 1{,}341 & 0.63 & 0.20 & 53.61 \\
    \texttt{OctaonOcta} & $O$ & 24 & 172 & -- & 0.68 & 0.61 & -- \\
    \bottomrule
  \end{tabular*}
\end{table}

\section{The Augmented Selective Spherical Bispectrum}
\label{sec:so3-selective}

We construct the augmented selective spherical bispectrum
$\Phi_{\mathrm{sel}}$, an
$\SO(3)$-invariant of size $\Theta(L^2)$ on band-limited spherical
signals, and probe its completeness empirically.

\paragraph{Setup and construction.} As in Section~\ref{sec:background}, a real signal
$f : S^2 \to \R$ band-limited at degree~$L$ has coefficient vectors
$\mathbf{F}_\ell \in \C^{2\ell+1}$, $\ell=0,\dots,L$, that transform as
$\mathbf{F}_\ell \mapsto D^\ell(g)\,\mathbf{F}_\ell$ under $g\in\SO(3)$. The reality
condition leaves $(L{+}1)^2$ real degrees of freedom (Appendix~\ref{sec:app-setup}), and the
full scalar bispectrum~\eqref{eq:bispectrum-sphere} has $\Theta(L^3)$ entries. The orbit space
$\R^{(L+1)^2}/\SO(3)$ has dimension $(L{+}1)^2 - 3$. By a rank
argument, any smooth complete $\SO(3)$-invariant must therefore have
at least $\Omega(L^2)$ components
(Proposition~\ref{prop:app-lower-bound}). 

\paragraph{The augmented selective spherical bispectrum.} We define
$\Phi_{\mathrm{sel}}(f)$ as the output of
Algorithm~\ref{alg:selective-bsp-so3}. It has two parts: a
\emph{bispectral part} $\mathcal{S}_\beta$, a selected subset of the
coefficients~\eqref{eq:bispectrum-sphere}, and a \emph{CG-power part}
$\mathcal{S}_P$ of scalars
$P_{\ell_1,\ell_2,\ell} := \lVert(\mathbf{F}_{\ell_1}\otimes
\mathbf{F}_{\ell_2})|_\ell\rVert^2$, the squared norms of
Clebsch--Gordan projections (Definition~\ref{def:app-cg-power}).
Both are built degree by degree. Once the lower degrees are known, an entry
$\beta_{a,b,\ell}$ with $a,b<\ell$ is linear in $\mathbf{F}_\ell$, so the entries of
degree $\ell$ form a system in the $2\ell{+}1$ real unknowns of $\mathbf{F}_\ell$.
Degrees $\ell\le 4$ contribute 22 bispectral entries and degrees $5\le\ell\le 7$ the fixed
lists $\mathcal{T}^{\mathrm{small}}_\ell$ (Table~\ref{tab:Tsmall}). Each degree $\ell\ge 8$
contributes up to $2\ell{+}1$ entries $\beta_{a,b,\ell}$ with $1\le a\le b<\ell$,
one per unordered triple, since on real signals all orderings of a
triple are proportional (Remark~\ref{rem:app-Tell-rank-deficit}). At $\ell\in\{8,9\}$,
where only 16 such triples are nonzero, the remainder is filled by the power entry
$\beta_{0,\ell,\ell}=a_0^0\lVert\mathbf{F}_\ell\rVert^2$ and the self-couplings
$\beta_{\ell,\ell,\ell'}$, which are nonlinear in $\mathbf{F}_\ell$. Every degree $\ell\ge 3$
also receives the self-couplings with even $\ell'\le\ell$ and a fixed family of at most seven
CG-power scalars. Odd $\ell'$ is skipped because an entry with odd $\ell_1{+}\ell_2{+}\ell$
and a repeated index vanishes identically on real signals
(Proposition~\ref{prop:app-odd-vanishing}). With these included, the Jacobian of the
degree-$\ell$ entries with respect to the $2\ell{+}1$ real unknowns of $\mathbf{F}_\ell$ has
full rank $2\ell{+}1$ at random signals, so $\mathbf{F}_\ell$ is locally determined by the
lower degrees. The test suite checks this for $4\le\ell\le 20$ and a verification script for
$\ell\le 100$. Full rank is a local property and does not exclude the mirror image $T_R f$
of $f$, which is generally not a rotation of $f$. Entries with even $\ell_1{+}\ell_2{+}\ell$
and all CG powers take the same value on $f$ and $T_R f$. Entries with odd
$\ell_1{+}\ell_2{+}\ell$ change sign, but they vanish identically up to degree 3, so the
first that tells $f$ from $T_R f$ is $\beta_{2,3,4}$, which $\Phi_{\mathrm{sel}}$ includes
(Appendix~\ref{sec:app-parity}). The total size is $\Theta(L^2)$
(Proposition~\ref{prop:app-output-size}): $380$ (resp.\ $426$) real
scalars at $L=15$ (resp.\ $16$), against $1{,}124$ (resp.\
$1{,}341$) entries of the full bispectrum
(Table~\ref{tab:app-counts}).

\begin{algorithm}[H]
\caption{Index set of $\Phi_{\mathrm{sel}}$, degree by degree. The
  low-degree and CG-power details are in
  Algorithm~\ref{alg:so3-augmented}
  (Appendix~\ref{sec:app-selective-index}).}
\label{alg:selective-bsp-so3}
\small
\begin{algorithmic}[1]
\STATE \textbf{Input:} band limit $L$
\STATE $(\mathcal{S}_\beta,\mathcal{S}_P) \leftarrow$ entries of
  degree $\le\min(L,4)$ \COMMENT{22 bispectral, 10 CG-power at $L\ge4$}
\FOR{$\ell = 5, \dotsc, L$}
  \STATE $\mathcal{B}_\ell \leftarrow \mathcal{T}^{\mathrm{small}}_\ell$
    if $\ell\le 7$, else $(a,b,\ell)$ for $1\le a\le b<\ell$,
    $a{+}b\ge\ell$, by $a{+}b$ then $a$
  \STATE append to $\mathcal{S}_\beta$ the first $2\ell{+}1$ entries of
    $\mathcal{B}_\ell$, $(0,\ell,\ell)$, $(\ell,\ell,\ell')$ for even
    $\ell'<\ell$ descending, that do not vanish on real signals
  \STATE append $(\ell,\ell,\ell')$ to $\mathcal{S}_\beta$ for even
    $2\le\ell'\le\ell$, if not present
  \STATE append the fixed CG-power triples of degree $\ell$ to
    $\mathcal{S}_P$ \COMMENT{5 for even $\ell$, 7 for odd}
\ENDFOR
\RETURN $\{\beta_t(f)\}_{t\in\mathcal{S}_\beta}\cup
  \{P_t(f)\}_{t\in\mathcal{S}_P}$
\end{algorithmic}
\end{algorithm}

\paragraph{Empirical evidence of completeness.}
\label{sec:so3-recon}
We call $\Phi_{\mathrm{sel}}$ a \emph{complete} $\SO(3)$-invariant if
$\Phi_{\mathrm{sel}}(f) = \Phi_{\mathrm{sel}}(g)$ implies that $f$ and $g$ lie in the same
orbit. Completeness can be probed empirically by attempting to invert $\Phi_{\mathrm{sel}}$:
if it is complete, any
signal $\hat f$ matching $\Phi_{\mathrm{sel}}(f)$ must agree
with $f$ up to a global rotation, so consistent recovery of
the input is evidence that no orbit information is lost.
We run this test at the classifier's band limit $L{=}15$, where
$\Phi_{\mathrm{sel}}$ has 380 coefficients, on 8 Spherical MNIST digits
(each in its original and one random orientation) and 16 random
band-limited signals, sampled on a $128{\times}256$ grid. A
degree-by-degree solver works on the 256 real spherical-harmonic
coefficients: Levenberg--Marquardt with the exact Jacobian fits
degrees 0--4 from random starts, then adds one degree at a time and
refits all degrees so far. Each solution is rotated onto its target
before the relative image error is measured
(Appendix~\ref{app:so3-recon}).

\begin{figure}[t]
  \centering
  \includegraphics[width=0.9\linewidth]{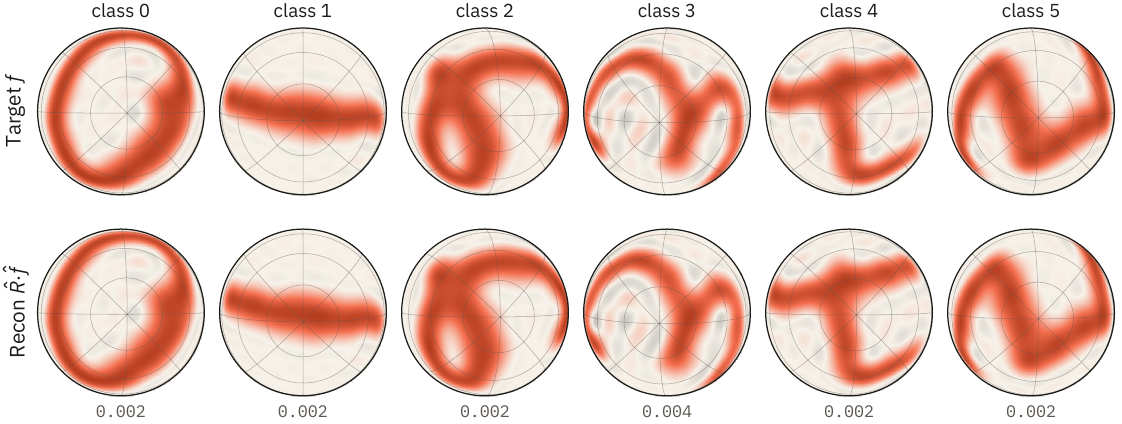}
  \caption{\textbf{Spherical signals recovered from their 380 augmented selective spherical
    bispectrum coefficients alone match the targets up to rotation.}
    Top: Spherical MNIST digits, band-limited to $L{=}15$. Bottom: reconstructions by a
    degree-by-degree solver from random starts, rotated onto the target (relative image error
    below each). All 16 targets (8
    digits, each in 2 orientations) are recovered: invariants match to better than
    $10^{-13}$ and images to a median $1.8\times10^{-3}$, the error left by interpolating the
    alignment rotation on the grid. Gradient descent from noise recovers 9 of 16
    (Appendix~\ref{app:so3-recon}).}
  \label{fig:so3-reconstruction}
\end{figure}

\paragraph{Results.}
All 32 targets are recovered: the augmented selective spherical bispectrum matches to
better than $10^{-13}$, the full bispectrum to $10^{-12}$, and the
aligned images to $0.0014$--$0.0035$ (digits) and $0.0056$--$0.0070$
(random signals), the error left by interpolating the alignment rotation on the grid
(Figure~\ref{fig:so3-reconstruction}). No run found two signals in
different orbits with equal invariants. Gradient descent on the image from noise recovers 9 of 16 digits and none of the random signals, on which it
stalls at a $7.5\%$ median invariant error, so the failures come from the optimizer.
This does not prove completeness. We conjecture that
$\Phi_{\mathrm{sel}}$ is generically complete, that is, complete outside a measure-zero set
of signals.

\section{Experiments}
\label{sec:experiments}

We evaluate the selective $G$-bispectrum as the invariant pooling layer of $G$-equivariant
networks on two benchmarks, 2D histopathology under $C_8$ and 3D organ classification under
the octahedral group $O$ (Section~\ref{sec:pooling}), and the augmented selective spherical bispectrum
on spherical digit classification (Section~\ref{sec:smnist}).

\subsection{$G$-bispectral pooling in $G$-CNNs}
\label{sec:pooling}

\paragraph{Setup.}
PatchCamelyon~\citep{veeling2018rotation} (PCam) has 327K $96{\times}96$ histopathology
patches with binary metastasis labels. Every model uses a $C_8$-equivariant
DenseNet~\citep{huang2017densenet} backbone following \citet{veeling2018rotation}, and we
report test AUC. OrganMNIST3D~\citep{medmnistv2} has 1{,}742 abdominal CT volumes ($28^3$
voxels, 11 classes, 971 for training). Every model uses an $O$-equivariant 3D ResNet with
exact voxel-permutation kernels, and we report test accuracy. Within a dataset the
equivariant models differ only in the terminal invariant map over the group dimension of the
feature maps: a maximum (Max pool, and the NormReLU and Gated~\citep{weiler2019general} and
Fourier-ELU~\citep{cesa2022program} backbones, named after their nonlinearity,
Appendix~\ref{app:pooling}), the $\ell_2$ norm (Norm pool),
or the selective $G$-bispectrum (128 coefficients for \texttt{CnonCn} with $n{=}8$ and 172 for
\texttt{OctaonOcta}). The bispectral head applies only fixed injective maps (real and
imaginary parts, signed logarithm) before a learned $1{\times}1$ projection and a ReLU. A
non-equivariant CNN trained with augmentation over $G$ (Aug.\ CNN) is the baseline without
built-in symmetry. Equivariant models are trained on canonical data only and tested on the
test set transformed by every element of $G$ (the \emph{rotated test}). Every model is trained
with AdamW~\citep{loshchilov2019adamw} for a fixed number of epochs that ends with a cosine
cycle~\citep{loshchilov2017sgdr}, without early stopping, and evaluated at its
best-validation checkpoint. We vary the training-set size
(100 to 262K examples on PCam and 50 to 971 on OrganMNIST3D) and the width (${\sim}30$K to
1.6M parameters on PCam and 0.37M to 6.3M for the equivariant models on OrganMNIST3D), over 3 to 10 seeds
(Appendices~\ref{app:pooling-grid} and~\ref{app:exp}).

\begin{figure}[t]
  \centering
  \includegraphics[width=\linewidth]{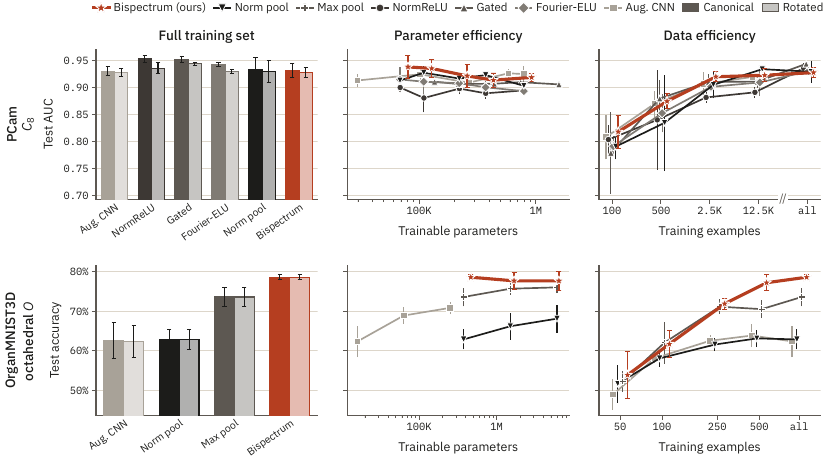}
  \caption{\textbf{Selective $G$-bispectral pooling matches baselines on PCam ($G{=}C_8$) and
    beats max pooling on OrganMNIST3D ($G{=}O$) at full data.}
    Top row, PCam (test AUC). Bottom row, OrganMNIST3D (test accuracy).
    Left column, full training set, with canonical (dark) and rotated (light) test bars. All
    models are on par on PCam, and on OrganMNIST3D bispectral pooling beats max pooling on the
    same backbone by 5.0\,pp at the smallest width.
    Middle column, rotated-test metric against trainable parameters. On OrganMNIST3D bispectral
    pooling stays above max pooling at every width (not significantly at the largest), and its
    smallest model beats max pooling's largest.
    Right column, rotated-test metric against training-set size. On PCam no model is
    significantly better than bispectral pooling at any size. On OrganMNIST3D its lead over max
    pooling appears from 500 examples, where it already beats max pooling trained on all 971.
    The rotated test averages the metric over all $|G|$ rotations of the test set. Equivariant
    models are trained on canonical data only, and Aug.\ CNN is trained with augmentation over
    $G$. Mean $\pm$ std over 3--10 seeds, with details in Appendix~\ref{app:pooling-grid}.}
  \label{fig:pooling-grid}
\end{figure}

\paragraph{Results.}
Figure~\ref{fig:pooling-grid} shows the results. Across 46 seed-level comparisons
(Welch $t$-tests on the rotated test against every baseline at every training-set size, and at
every width on OrganMNIST3D, detailed in Appendix~\ref{app:pooling-grid}),
bispectral pooling is never significantly worse than any baseline. On OrganMNIST3D it
improves on max pooling on the same backbone by 5.0\,pp at full data
($78.54\%$ vs.\ $73.56\%$, $p<10^{-3}$). Its smallest model (0.46M parameters) beats max
pooling's largest (6.0M) by 2.6\,pp, and with 500 training examples it beats max pooling
trained on all 971 by 3.6\,pp. It is 3.6 to 16\,pp above norm pooling and the augmented 3D
CNN at every training-set size from 100 and at every width. The advantage over max pooling
vanishes at the smallest training sets ($n\le250$: $-0.3$ to $+1.6$\,pp, not significant) and is not significant at
the largest width ($+1.6$\,pp, $p{=}0.10$). On PCam all methods are on par: no model is
significantly better than bispectral pooling at any training-set size. With 3 seeds per arm,
these tests have 80\% power only for gaps of about 0.025 AUC. All equivariant models score the same on the
canonical and rotated tests of OrganMNIST3D, where the 24 rotations permute voxels exactly.
Placing the ReLU before the projection instead zeroes every
negative coefficient and costs 4--5\,pp on OrganMNIST3D (Appendix~\ref{app:pooling-head}).

\subsection{Spherical MNIST: $\SO(3)$ invariance on the sphere}
\label{sec:smnist}

The pooling experiments cover finite groups. For a continuous group, we test \texttt{SO3onS2} (the augmented selective spherical
bispectrum, Section~\ref{sec:so3-selective}) on Spherical
MNIST~\citep{cohen2018spherical}, the standard benchmark for
rotation invariance on $S^2$.

\paragraph{Setup.}
MNIST digits are stereographically projected onto a
$64\times 128$ equiangular grid. We compute
$\Phi_{\mathrm{sel}}$ at $L{=}15$ ($N=380$ real invariants, stored as complex numbers by the library), apply
sign-preserving $\operatorname{sign}(x)\log(1+|x|)$ to real and
imaginary parts ($2N=760$ real channels, keeping both channels for simplicity even though one channel per
bispectral entry is identically zero by parity), and
feed a 3-layer MLP (230K parameters). We compare against two learned
$\SO(3)$-equivariant networks run in the same pipeline and matched in size: an
S2CNN~\citep{cohen2018spherical} with BatchNorm (155K parameters) and a
Clebsch--Gordan tensor-product network built with \texttt{e3nn}~\citep{geiger2022e3nn}
(386K). Controls are a power-spectrum MLP (11K parameters, equally invariant but
\emph{incomplete}, discarding the relative phases between degrees) and a non-invariant CNN on the
equirectangular image (185K). Following \citet{cohen2018spherical}, we evaluate over 5 seeds under the
protocols NR/NR, R/R, and NR/R (train unrotated, test rotated), which tests invariance to
unseen rotations. The bispectrum and the two equivariant networks are each tuned on validation
accuracy at full data, and the selected settings are reused at every training-set size.

\begin{figure}[tbp]
  \centering
  \includegraphics[width=\linewidth]{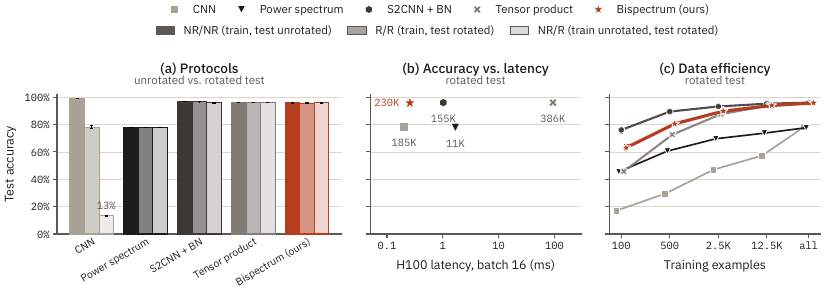}
  \caption{\textbf{The augmented selective spherical bispectrum is within 0.3\,pp of equivariant networks
    on Spherical MNIST at $4\times$ to $364\times$ lower latency.}
    (a)~Accuracy under the three train/test protocols (NR: unrotated, R: rotated). The CNN
    drops to 13\% on unseen rotations, while the invariant and equivariant models do not.
    (b)~Rotated-test accuracy against forward latency (spherical grid in, logits out, float32,
    batch 16, H100, median), with labels giving trainable parameters. (c)~Rotated-test accuracy
    against training-set size. Bispectrum + MLP trails S2CNN + BatchNorm below the full set and
    beats the tensor-product net up to $n{=}2{,}500$. Mean $\pm$ std over 5 seeds, with
    settings in Table~\ref{tab:hyperparams}.}
  \label{fig:smnist}
\end{figure}

\paragraph{Results.}
Figure~\ref{fig:smnist} reports the results. On the
rotated test, Bispectrum + MLP reaches $95.8\%$, on par with the matched S2CNN
($96.1\%$, $p{=}0.14$) and $0.25$\,pp below the tensor-product network ($p{=}0.006$), at
$4\times$ and $364\times$ lower latency. Like the tensor-product network it loses nothing
from NR/NR to NR/R, whereas S2CNN, which leads on NR/NR ($97.1\%$) and R/R ($96.6\%$), loses
about 1\,pp. The planar CNN collapses from $99.1\%$ to $13.4\%$. The power spectrum is equally
invariant but reaches only $77.7\%$, and scaling its MLP $48\times$ to 550K parameters lifts
it only to $80.1\%$, below the bispectrum with 25K parameters ($94.1\%$,
Appendix~\ref{app:smnist-capacity}), so more capacity does not close the gap. With
less data, the learned spherical CNN is more sample-efficient: S2CNN leads at every
training-set size below the full set (by 13.0\,pp at $n{=}100$ and 1.3\,pp at 12.5K, Holm
$p<10^{-3}$), while Bispectrum + MLP beats the tensor-product network up to $n{=}2{,}500$
(by 17.3\,pp at $n{=}100$) and the power spectrum at every size.

\section{Discussion and Conclusion}
\label{sec:discussion}

We have presented \texttt{bispectrum}, a PyTorch library that makes
selective $G$-bispectral invariants practical, covering seven
group actions through one \texttt{nn.Module} interface.
We contribute the augmented selective spherical bispectrum, a $\Theta(L^2)$ $\SO(3)$ invariant on
$S^2$ (Section~\ref{sec:so3-selective}), which a degree-by-degree solver inverts on every
signal we tested. We conjecture that it is generically complete. The selective forward pass
takes under a millisecond on GPU for every module at batch 16. As a pooling layer in
$G$-CNNs, the selective $G$-bispectrum is never significantly worse than an augmented CNN, than
norm or max pooling, or than the NormReLU, Gated, and Fourier-ELU backbones. It is on par on PCam and improves on max
pooling by 5.0\,pp on OrganMNIST3D. On Spherical MNIST, the augmented selective spherical bispectrum with an
MLP head is within 0.3\,pp of matched learned-equivariant networks on rotated digits, at
$4\times$ to $364\times$ lower latency. We speculate that the gain grows with how much
information incomplete pooling discards. Per channel, norm pooling keeps a single value, the
total energy, and max pooling keeps 1 of $|G|$ values, so both discard more for $O$ ($|G|{=}24$)
than for $C_8$ ($|G|{=}8$). Within the pooling head, a ReLU before the first learned
projection costs 4--5\,pp on OrganMNIST3D (Appendix~\ref{app:pooling-head}).

\paragraph{When not to use the bispectrum.}
Each module is invariant to a \emph{specific} group only. Data augmentation or
flexible-symmetry equivariant networks fit better when the relevant symmetry is multi-group or
unknown. On the sphere, a fixed invariant followed by an MLP is less sample-efficient than a
learned spherical CNN.

\paragraph{Limitations and future work.}
The degree-by-degree inversion of the augmented selective spherical bispectrum is in our experiment code and not yet in the library, and a formal completeness proof, a characterization of the signals on which completeness fails,
and numerical-conditioning analysis at large band limits remain open. On $S^2$ we compare
against one tensor-product network, not against MACE~\citep{batatia2022mace} or wavelet
scattering, and use the invariant with an MLP head. Bispectral pooling at the end of an
equivariant spherical backbone is future work.

\begin{ack}
A.~M.~L.\ and N.~M.\ acknowledge partial support from NSF grant 2313150, NSF CAREER grant
2240158, and the Noyce Foundation. A.~M.~L.\ also acknowledges support from an NSF Graduate
Research Fellowship (GRFP) and NSF Grant No. DGE-2125644. S.~M.\ is a Research Fellow of the
Fonds de la Recherche Scientifique (FNRS). J.~M.\ thanks his Atmo colleagues Giorgio Sarro, Duncan Axisa, Yves
Martin, Alexander Sviridenko, Sander Henning, and Deepika Misra, as well as Creon Levit, for their
feedback. The authors declare no competing interests.
\end{ack}

\bibliographystyle{plainnat}
\bibliography{references}

\appendix

\section{Augmented Selective Spherical Bispectrum -- Structural Background}
\label{app:so3-proof}

This appendix gives the structural background for the construction
of Section~\ref{sec:so3-selective}. Section~\ref{sec:app-setup} fixes the $S^2$ band-limited
setup and notation. Section~\ref{sec:app-lower-bound} derives the orbit-space dimension
lower bound $(L{+}1)^2-3$ that motivates the $\Theta(L^2)$ target.
Section~\ref{sec:app-selective-index} gives the index set as implemented
(Algorithm~\ref{alg:so3-augmented}) and its size, and Section~\ref{sec:app-parity} the parity
structure of the scalar bispectrum that determines which triples vanish
identically on real signals. Section~\ref{sec:app-comparison} compares the augmented selective spherical bispectrum with prior
bispectral constructions, and Section~\ref{sec:app-pixel-complexity} counts its cost in
pixel space. We do not claim a completeness proof. Generic completeness
is a conjecture, supported by the full rank of the per-degree Jacobian
(Section~\ref{sec:so3-selective}) and by
the reconstructions of Figure~\ref{fig:so3-reconstruction}.

\subsection{Setup and notation}
\label{sec:app-setup}

\begin{definition}[Band-limited signal on $S^2$]
A signal $f : S^2 \to \R$ is \emph{band-limited} at degree $L$ if its
spherical harmonic expansion
\begin{equation}\label{eq:app-sh-expansion}
  f(\theta, \varphi)
  = \sum_{\ell=0}^{L} \sum_{m=-\ell}^{\ell}
    a_\ell^m \, \SH_\ell^m(\theta, \varphi)
\end{equation}
has $a_\ell^m = 0$ for all $\ell > L$. We write
$\mathbf{F}_\ell = (a_\ell^{-\ell}, \dotsc, a_\ell^{\ell})
  \in \C^{2\ell+1}$
for the coefficient vector at degree~$\ell$.
\end{definition}

\noindent
For real $f$, the conjugation symmetry
$a_\ell^{-m} = (-1)^m \conj{a_\ell^m}$ holds, so
$\mathbf{F}_\ell$ has $2\ell+1$ real degrees of freedom.
The total signal dimension is
$\dim = \sum_{\ell=0}^{L}(2\ell+1) = (L+1)^2$.

\begin{definition}[Rotation action]
Under $g \in \SO(3)$, the spherical harmonic coefficients transform as
$\mathbf{F}_\ell \mapsto D^\ell(g)\,\mathbf{F}_\ell$,
where $D^\ell$ is the $(2\ell+1)$-dimensional Wigner~$D$-matrix.
\end{definition}

\noindent
In the irrep notation of~\citet{mataigne2024selective},
$\{\rho_\ell\}_{\ell\ge 0}$ are the (unitary) irreducible
representations of $\SO(3)$ with $V_{\rho_\ell}$ the
$(2\ell{+}1)$-dimensional space of degree-$\ell$ spherical harmonics
and $\rho_\ell(g)=D^\ell(g)$. The Fourier coefficient of $f$ at
$\rho_\ell$ is $\mathcal{F}(f)_{\rho_\ell}=\mathbf{F}_\ell$. We use
$\mathbf{F}_\ell$ as the working symbol throughout.

\begin{definition}[Bispectrum on $S^2$]\label{def:app-bispectrum}
For each triple $(\ell_1, \ell_2, \ell)$ satisfying the triangle
inequality $\abs{\ell_1 - \ell_2} \le \ell \le \ell_1 + \ell_2$,
the \emph{bispectral coefficient} is
\begin{equation}\label{eq:app-bispectrum}
  \beta_{\ell_1,\ell_2,\ell}
  = \sum_{\substack{m_1,\,m_2,\,m \\ m_1+m_2=m}}
    \CG{\ell_1}{m_1}{\ell_2}{m_2}{\ell}{m}\;
    a_{\ell_1}^{m_1}\,a_{\ell_2}^{m_2}\,\conj{a_\ell^m},
\end{equation}
where $\CG{\ell_1}{m_1}{\ell_2}{m_2}{\ell}{m}$ denotes the
Clebsch--Gordan coefficient
$\langle \ell_1, m_1;\, \ell_2, m_2 \mid \ell, m \rangle$.
Each $\beta_{\ell_1,\ell_2,\ell}$ is a scalar. For real~$f$, it is
real when $\ell_1{+}\ell_2{+}\ell$ is even and purely imaginary when
$\ell_1{+}\ell_2{+}\ell$ is odd (Lemma~\ref{lem:app-parity}).
\end{definition}

Definition~\ref{def:app-bispectrum} adapts the $G$-bispectrum~\eqref{eq:bispectrum-def}
to $\SO(3)$ acting on~$S^2$: the irreps are Wigner-$D$ matrices, the Fourier coefficients are the vectors
$\mathbf{F}_\ell$ rather than square matrices, and contracting with the Clebsch--Gordan
coefficients leaves one scalar per triple.

\begin{definition}[CG power spectrum]\label{def:app-cg-power}
For a triple $(\ell_1, \ell_2, \ell)$ satisfying the triangle
inequality, write
$(\mathbf{F}_{\ell_1} \otimes \mathbf{F}_{\ell_2})\big|_\ell$
for the projection of the Kronecker product
$\mathbf{F}_{\ell_1} \otimes \mathbf{F}_{\ell_2}$ onto the
component that transforms by $D^\ell$, realized by contracting with the CG
coefficients $\CG{\ell_1}{m_1}{\ell_2}{m_2}{\ell}{m}$. The
\emph{CG power spectrum entry} is the squared norm of this projection,
\begin{equation}\label{eq:app-cg-power}
  P_{\ell_1,\ell_2,\ell}
  := \lVert (\mathbf{F}_{\ell_1} \otimes \mathbf{F}_{\ell_2})|_\ell
     \rVert^2
  = \sum_{m=-\ell}^{\ell}
    \Bigl\lvert \sum_{\substack{m_1,m_2 \\ m_1+m_2=m}}
      \CG{\ell_1}{m_1}{\ell_2}{m_2}{\ell}{m}\;
      a_{\ell_1}^{m_1}\, a_{\ell_2}^{m_2} \Bigr\rvert^2.
\end{equation}
This is a degree-$4$, real-valued, $\SO(3)$-invariant polynomial in
the signal coefficients. The CG-power part $\mathcal{S}_P$ of the
augmented selective spherical bispectrum is a fixed family of these scalars, at
most seven per degree and $O(L)$ in total ($77$ at $L=15$), listed in
Algorithm~\ref{alg:so3-augmented}. With them included, the per-degree Jacobian of
Section~\ref{sec:so3-selective} has full rank $2\ell{+}1$.
\end{definition}

\begin{proposition}[Size of the full bispectrum]\label{prop:app-full-size}
The number of admissible triples $(\ell_1, \ell_2, \ell)$ with
$0 \le \ell_1, \ell_2, \ell \le L$ satisfying the triangle inequality
$|\ell_1 - \ell_2| \le \ell \le \ell_1 + \ell_2$ is $\Theta(L^3)$.
\end{proposition}
\begin{proof}
For fixed $\ell_1$ and $\ell_2$, the triangle inequality admits
$\min(\ell_1 + \ell_2, L) - |\ell_1 - \ell_2| + 1$ valid values of
$\ell$. Summing over $(\ell_1,\ell_2)\in[0,L]^2$ gives
$N_{\mathrm{full}} = \Theta(L^3)$.
\end{proof}

\begin{remark}[Invariance]
Each $\beta_{\ell_1,\ell_2,\ell}$ is $\SO(3)$-invariant: the Wigner
matrices satisfy $D^{\ell_1} \otimes D^{\ell_2}
= \bigoplus_\ell D^\ell$ under the CG decomposition, so the triple
product contracted with CG coefficients is unchanged by
$\mathbf{F}_\ell \mapsto D^\ell(g)\mathbf{F}_\ell$.
\end{remark}

\subsection{Lower bound on the output size}
\label{sec:app-lower-bound}

\begin{proposition}[Dimensional lower bound]\label{prop:app-lower-bound}
Any \emph{smooth or polynomial} complete $\SO(3)$-invariant for
band-limited signals on $S^2$ at band limit $L$ must have codomain
of dimension at least $(L+1)^2 - 3$.
\end{proposition}
\begin{proof}
Let $\Phi : \R^{(L+1)^2} \to \R^k$ be a smooth (or polynomial) $\SO(3)$-invariant map that
separates orbits on an open dense subset $U\subset\R^{(L+1)^2}$, and let
$r := \max_{x\in U} \mathrm{rank}\, d\Phi_x$. The rank of $d\Phi$ is lower semicontinuous,
so $W := \{x\in U : \mathrm{rank}\, d\Phi_x = r\}$ is open and nonempty. By the constant-rank
theorem, for $x\in W$ the level set $\Phi^{-1}(\Phi(x))\cap W$ is, near $x$, a submanifold of
dimension $(L+1)^2 - r$. Orbit separation on $U$ places it inside the orbit $\SO(3)\cdot x$,
a manifold of dimension at most $\dim\SO(3) = 3$. Hence $(L+1)^2 - r \le 3$, and
$k \ge r \ge (L+1)^2 - 3 = \Omega(L^2)$.
\end{proof}

\begin{remark}[Smoothness assumption]
The smoothness assumption is essential. Without continuity, a single real number can separate
orbits (e.g., a Borel injection of the orbit space into $\R$), but such an invariant is of no use
in a network. The bispectral and Clebsch--Gordan power invariants used here are polynomial in
the coefficients, so the proposition applies.
\end{remark}

\subsection{The selective index set}
\label{sec:app-selective-index}

This subsection describes the index set as implemented in \texttt{bispectrum.so3\_on\_s2}, which
Algorithm~\ref{alg:so3-augmented} mirrors. The \emph{degree}
of an entry is its largest index. We call
$\beta_{0,\ell,\ell}=a_0^0\lVert\mathbf{F}_\ell\rVert^2$ the \emph{power entry} of degree
$\ell$ and $\beta_{\ell,\ell,\ell'}$ with $\ell'\le\ell$ its \emph{self-couplings}.

\begin{definition}[Augmented selective spherical bispectrum]\label{def:app-selective}
The \emph{augmented selective spherical bispectrum} $\Phi_{\mathrm{sel}}$ at
band limit $L$ consists of a bispectral part $\mathcal{S}_\beta$
(entries $\beta_{\ell_1,\ell_2,\ell}$) and a CG-power part
$\mathcal{S}_P$ (entries $P_{\ell_1,\ell_2,\ell}$,
Definition~\ref{def:app-cg-power}), assembled degree by degree.
Degrees $\ell\le\min(L,7)$ use the fixed entries of
Table~\ref{tab:Tsmall}. Each degree $\ell\ge 8$ receives
\begin{enumerate}[nosep,label=(\roman*)]
\item the first $2\ell{+}1$ entries, among those that do not vanish
  identically on real signals
  (Proposition~\ref{prop:app-odd-vanishing}), of the list made of the
  triples $(a,b,\ell)$ with $1\le a\le b\le\ell{-}1$ and
  $a{+}b\ge\ell$, ordered by $a{+}b$ and then by $a$, followed by the power entry
  $(0,\ell,\ell)$ and the self-couplings $(\ell,\ell,\ell')$ for even
  $\ell'<\ell$ in descending order,
\item every even self-coupling $(\ell,\ell,\ell')$,
  $2\le\ell'\le\ell$, not already selected, and
\item five (even $\ell$) or seven (odd $\ell$) CG-power triples
  $(\ell_1,\ell,\ell')$ with $\ell_1\le 3$
  (Algorithm~\ref{alg:so3-augmented}).
\end{enumerate}
At $\ell\in\{8,9\}$ the list in (i) has only 16 nonzero triples
$(a,b,\ell)$, so the power entry and self-couplings make up the remaining entries. From $\ell=10$ on, the triples alone supply
all $2\ell{+}1$. Only one ordering
of each triple is listed because, on real signals, all orderings of a
triple are proportional
(Remark~\ref{rem:app-Tell-rank-deficit}).

Mathematically, $\Phi_{\mathrm{sel}}:V_L\to\R^N$, where
$V_L\cong\R^{(L+1)^2}$ is the space of real signals band-limited
at~$L$, has $N$ real components: by
Lemma~\ref{lem:app-parity}, each bispectral entry contributes one
real component on real signals (real for even parity,
purely imaginary for odd), and CG-power entries are real-valued
degree-$4$ invariants. The PyTorch implementation stores each
scalar in a complex tensor, so downstream layers receive $2N$ real
channels after splitting real and imaginary parts. For bispectral
entries, exactly one of the two channels is identically zero by
parity, and for CG-power entries the imaginary channel is
identically zero.
\end{definition}

\begin{table}[h]
  \caption{Entries of $\Phi_{\mathrm{sel}}$ at degrees $\ell\le 7$, as
    implemented in the library. Each triple $(\ell_1,\ell_2,\ell_3)$ denotes
    $\beta_{\ell_1,\ell_2,\ell_3}$ in the bispectral columns and
    $P_{\ell_1,\ell_2,\ell_3}$ in the CG-power column. The capped column holds the entries
    counted against the budget of $2\ell{+}1$ per degree ($10$ at $\ell=4$), and the added
    column holds the even self-couplings appended beyond it. For
    $4\le\ell\le 7$ the capped entries are the fixed lists
    $\mathcal{T}_\ell^{\mathrm{small}}$ of triples $(a,b,\ell)$ and reorderings $(a,\ell,b)$
    with $a,b<\ell$, without the identically-zero triples
    $(4,5,4)$ and $(6,7,6)$, padded with the power entry $(0,\ell,\ell)$ when the list has
    fewer than $2\ell{+}1$ entries. In total, there are 22 bispectral and 10 CG-power entries at
    $\ell\le4$, and 69 and 29 at $\ell\le7$.}
  \label{tab:Tsmall}
  \centering
  \footnotesize
  \setlength{\tabcolsep}{3pt}
  \begin{tabular}{@{}c>{\raggedright\arraybackslash}p{0.47\linewidth}>{\raggedright\arraybackslash}p{0.13\linewidth}>{\raggedright\arraybackslash}p{0.3\linewidth}@{}}
    \toprule
    $\ell$ & Capped bispectral & Added self-couplings & CG-power \\
    \midrule
    $0$ & $(0,0,0)$ & -- & -- \\
    $1$ & $(0,1,1)$ & -- & -- \\
    $2$ & $(1,1,2)$, $(0,2,2)$, $(2,2,2)$ & -- & $(1,2,1)$ \\
    $3$ & $(1,2,3)$, $(1,3,2)$, $(2,3,1)$, $(0,3,3)$, $(3,3,2)$ & -- & $(1,3,2)$, $(2,3,1)$, $(2,3,2)$, $(3,3,2)$ \\
    $4$ & $(1,3,4)$, $(2,2,4)$, $(2,3,4)$, $(3,3,4)$, $(1,4,3)$, $(2,4,2)$, $(3,4,1)$, $(2,4,3)$, $(3,4,2)$, $(3,4,3)$ & $(4,4,2)$, $(4,4,4)$ & $(1,4,3)$, $(1,4,4)$, $(2,4,2)$, $(2,4,3)$, $(3,4,1)$ \\
    $5$ & $(1,4,5)$, $(2,3,5)$, $(2,4,5)$, $(3,4,5)$, $(1,5,4)$, $(2,5,3)$, $(3,5,2)$, $(4,5,1)$, $(2,5,4)$, $(3,5,4)$, $(0,5,5)$ & $(5,5,2)$, $(5,5,4)$ & $(1,5,4)$, $(1,5,5)$, $(2,5,3)$, $(2,5,4)$, $(3,5,2)$, $(3,5,3)$, $(3,5,4)$ \\
    $6$ & $(1,5,6)$, $(2,4,6)$, $(3,3,6)$, $(3,4,6)$, $(1,6,5)$, $(2,6,4)$, $(3,6,3)$, $(4,6,2)$, $(5,6,1)$, $(2,6,5)$, $(3,6,5)$, $(4,6,5)$, $(5,6,5)$ & $(6,6,2)$, $(6,6,4)$, $(6,6,6)$ & $(1,6,5)$, $(1,6,6)$, $(2,6,4)$, $(2,6,5)$, $(3,6,3)$ \\
    $7$ & $(1,6,7)$, $(2,5,7)$, $(3,4,7)$, $(4,5,7)$, $(1,7,6)$, $(2,7,5)$, $(3,7,4)$, $(4,7,3)$, $(5,7,2)$, $(6,7,1)$, $(2,7,6)$, $(3,7,6)$, $(4,7,6)$, $(5,7,6)$, $(0,7,7)$ & $(7,7,2)$, $(7,7,4)$, $(7,7,6)$ & $(1,7,6)$, $(1,7,7)$, $(2,7,5)$, $(2,7,6)$, $(3,7,4)$, $(3,7,5)$, $(3,7,6)$ \\
    \bottomrule
  \end{tabular}
\end{table}

\begin{table}[h]
  \caption{Coefficient counts of $\Phi_{\mathrm{sel}}$ at
    representative band limits. ``Bispectral'' counts the triples in
    $\mathcal{S}_\beta$ (including the even self-couplings $\beta_{\ell,\ell,\ell'}$ that every degree
    $\ell\ge 3$ receives).
    ``CG-power'' counts entries in $\mathcal{S}_P$, and ``Total $N$'' is
    $|\Phi_{\mathrm{sel}}|$ (real components) and matches the
    library's \texttt{output\_size}. ``Stored channels'' ($2N$) counts the real and imaginary
    channels seen by the Spherical MNIST classifier.
    One channel per bispectral entry is identically zero by
    parity, and the imaginary channel of each CG-power entry is
    identically zero. ``Full'' counts the entries of the full
    bispectrum computed by \texttt{SO3onS2(selective=False)}: every
    admissible triple with $\ell_1\le\ell_2\le L$ and $\ell\le L$,
    including those that vanish identically on real signals.}
  \label{tab:app-counts}
  \centering
  \small
  \setlength{\tabcolsep}{5pt}
  \begin{tabular}{cccccc}
    \toprule
    $L$ & Bispectral & CG-power & Total $N=|\Phi_{\mathrm{sel}}|$
        & Stored channels ($2N$) & Full \\
    \midrule
    $4$  & $22$  & $10$  & $32$  & $64$  & $42$ \\
    $5$  & $35$  & $17$  & $52$  & $104$ & $69$ \\
    $15$ & $303$ & $77$  & $380$ & $760$ & $1{,}124$ \\
    $16$ & $344$ & $82$  & $426$ & $852$ & $1{,}341$ \\
    \bottomrule
  \end{tabular}
\end{table}

\begin{algorithm}[H]
\caption{Index set of the augmented selective spherical bispectrum at
  band limit~$L$, as implemented. Returns the bispectral triples
  $\mathcal{S}_\beta$ and the CG-power triples $\mathcal{S}_P$.}
\label{alg:so3-augmented}
\small
\begin{algorithmic}[1]
\STATE \textbf{Input:} band limit $L$
\STATE $\mathcal{S}_\beta \leftarrow$ bispectral entries of degree
  $\le\min(L,7)$ \COMMENT{Table~\ref{tab:Tsmall}}
\STATE $\mathcal{S}_P \leftarrow$ CG-power entries of
  degree $\le\min(L,7)$ \COMMENT{Table~\ref{tab:Tsmall}}
\FOR{$\ell = 8, \dotsc, L$}
  \STATE $C \leftarrow [(a,b,\ell) : 1\le a\le b\le\ell{-}1,\
    a{+}b\ge\ell]$, ordered by $a{+}b$, then $a$
  \STATE append $(0,\ell,\ell)$, then $(\ell,\ell,\ell')$ for even
    $\ell' = \ell{-}1, \dotsc, 2$, to $C$
  \STATE remove from $C$ the entries with odd $\ell_1{+}\ell_2{+}\ell_3$
    and a repeated index \COMMENT{$\equiv0$ on real signals,
    Prop.~\ref{prop:app-odd-vanishing}}
  \STATE $\mathcal{S}_\beta \leftarrow \mathcal{S}_\beta \cup$ first
    $\min(2\ell{+}1, |C|)$ entries of $C$
  \STATE $\mathcal{S}_\beta \leftarrow \mathcal{S}_\beta \cup
    \{(\ell,\ell,\ell') : \ell'\text{ even},\ 2\le\ell'\le\ell\}$
    \COMMENT{even self-couplings}
  \STATE $\mathcal{S}_P \leftarrow \mathcal{S}_P \cup
    \{(1,\ell,\ell{-}1),\, (1,\ell,\ell),\, (2,\ell,\ell{-}2),\,
    (2,\ell,\ell{-}1),\, (3,\ell,\ell{-}3)\}$
  \IF{$\ell$ is odd}
    \STATE $\mathcal{S}_P \leftarrow \mathcal{S}_P \cup
      \{(3,\ell,\ell{-}2),\, (3,\ell,\ell{-}1)\}$
  \ENDIF
\ENDFOR
\RETURN $(\mathcal{S}_\beta,\, \mathcal{S}_P)$
\end{algorithmic}
\end{algorithm}

\begin{proposition}[Output size]\label{prop:app-output-size}
As $L\to\infty$, $|\Phi_{\mathrm{sel}}| = \Theta(L^2)$. Each degree
$\ell\ge 8$ contributes at most $2\ell{+}1$ bispectral entries, at most $\lfloor\ell/2\rfloor$ further even self-couplings,
and at most seven CG-power entries. The degrees $\ell\le 7$
contribute a constant number (Table~\ref{tab:Tsmall}), and the first $2\ell{+}1$ entries of each degree alone
already number $\Theta(L^2)$. Per-$L$ counts are in
Table~\ref{tab:app-counts}.
\end{proposition}

\subsection{Parity structure of the bispectrum}
\label{sec:app-parity}

The selective construction of Section~\ref{sec:so3-selective}
exploits two structural facts about the scalar bispectrum
$\beta_{\ell_1,\ell_2,\ell}$: a parity rule that fixes whether each
entry is real or purely imaginary on real signals (Lemma~\ref{lem:app-parity})
and the resulting identical vanishing of certain odd-parity entries
with a repeated index (Proposition~\ref{prop:app-odd-vanishing}). Together with the
proportionality of the orderings of a triple, the lemma fixes how many independent equations
in $\mathbf{F}_\ell$ the bispectral entries of each degree supply (Remark~\ref{rem:app-Tell-rank-deficit}). The
proposition fixes which candidate entries Algorithm~\ref{alg:so3-augmented} discards.

Let $T_R$ be the reflection $(T_R\cdot f)(\theta,\varphi) :=
f(\theta,-\varphi)$ through the $xz$-plane, an element of
$\mathrm{O}(3)\setminus\SO(3)$. In the Condon--Shortley convention
behind the reality condition
$a_\ell^{-m}=(-1)^m\conj{a_\ell^m}$, it acts on coefficients as
$a_\ell^m\mapsto(-1)^m a_\ell^{-m}$. The signal $T_R\cdot f$ is the mirror image of $f$, which is
generally not a rotation of $f$, so a complete $\SO(3)$-invariant must tell them apart.

\begin{lemma}[Parity of the bispectrum]\label{lem:app-parity}
For any signal (real or complex):
\begin{equation}\label{eq:app-parity}
  \beta_{\ell_1,\ell_2,\ell}(T_R \cdot f)
  = (-1)^{\ell_1+\ell_2+\ell}\,\beta_{\ell_1,\ell_2,\ell}(f).
\end{equation}
For \emph{real} signals, there is an independent constraint:
\begin{equation}\label{eq:app-reality-constraint}
  \conj{\beta_{\ell_1,\ell_2,\ell}(f)}
  = (-1)^{\ell_1+\ell_2+\ell}\,\beta_{\ell_1,\ell_2,\ell}(f).
\end{equation}
Together:
\begin{enumerate}[nosep,label=(\alph*)]
\item If $\ell_1+\ell_2+\ell$ is even:
  $\beta$ is real-valued (by~\eqref{eq:app-reality-constraint})
  and $T_R$-invariant (by~\eqref{eq:app-parity}).
\item If $\ell_1+\ell_2+\ell$ is odd:
  $\beta$ is purely imaginary (by~\eqref{eq:app-reality-constraint})
  and $\beta(T_R \cdot f) = -\beta(f)$, so $\im\beta$ flips sign under $T_R$.
\end{enumerate}
\end{lemma}
\begin{proof}
Under $T_R$, each $a_\ell^m \mapsto (-1)^m a_\ell^{-m}$.
Substituting into $\beta = \sum_m (\sum_{m_1+m_2=m} C\, a^{m_1} a^{m_2})
\conj{a^m}$, the three sign factors $(-1)^{m_1+m_2+m} = (-1)^{2m} = 1$
cancel. Re-indexing $m \mapsto -m$ and applying the CG symmetry
$\CG{\ell_1}{-m_1}{\ell_2}{-m_2}{\ell}{-m}
  = (-1)^{\ell_1+\ell_2-\ell}\,\CG{\ell_1}{m_1}{\ell_2}{m_2}{\ell}{m}$
yields~\eqref{eq:app-parity} (with $(-1)^{\ell_1+\ell_2-\ell} = (-1)^{\ell_1+\ell_2+\ell}$).

For real signals, $\conj{a_\ell^m} = (-1)^m a_\ell^{-m}$.
Taking the complex conjugate of $\beta$ and re-indexing similarly
yields~\eqref{eq:app-reality-constraint}.
\end{proof}

\begin{lemma}[CG power is $T_R$-invariant]\label{lem:app-P-invariant}
For all triples $(\ell_1,\ell_2,\ell)$:
$P_{\ell_1,\ell_2,\ell}(T_R \cdot f) = P_{\ell_1,\ell_2,\ell}(f)$.
\end{lemma}
\begin{proof}
Under $T_R$, the CG-projected component transforms as
$(\mathbf{F}_{\ell_1} \otimes \mathbf{F}_{\ell_2})|_\ell^M
  \mapsto (-1)^{\ell_1+\ell_2+\ell+M}\,
  (\mathbf{F}_{\ell_1} \otimes \mathbf{F}_{\ell_2})|_\ell^{-M}$,
where the superscript $M$ selects the $m=M$ component.
Taking the squared norm and re-indexing the sum over $M$
leaves $P$ unchanged.
\end{proof}

\begin{proposition}[Vanishing of odd-parity entries with a repeated index]
\label{prop:app-odd-vanishing}
Let $\beta_{\ell_1,\ell_2,\ell}$ be the scalar bispectrum entry of a
real signal $f$ (so $\ell_1+\ell_2+\ell$ odd). If the index multiset
$\{\ell_1,\ell_2,\ell\}$ contains a repeated value, then
$\beta_{\ell_1,\ell_2,\ell}(f)\equiv 0$ as a polynomial in the SH
coefficients of $f$.
\end{proposition}

\begin{proof}
Define the symmetric scalar 3j contraction
\[
  T_{\ell_1,\ell_2,\ell_3}(f) :=
  \sum_{m_1,m_2,m_3}
  \begin{pmatrix} \ell_1 & \ell_2 & \ell_3 \\
                  m_1 & m_2 & m_3 \end{pmatrix}\,
  a_{\ell_1}^{m_1}\, a_{\ell_2}^{m_2}\, a_{\ell_3}^{m_3} .
\]
Converting the Wigner~3j into a Clebsch--Gordan coefficient via
$\bigl(\begin{smallmatrix} j_1 & j_2 & j_3 \\
       m_1 & m_2 & m_3\end{smallmatrix}\bigr)
= \frac{(-1)^{j_1-j_2-m_3}}{\sqrt{2j_3+1}}\,
  \CG{j_1}{m_1}{j_2}{m_2}{j_3}{-m_3}$ and applying the reality identity
$a_{\ell_3}^{m_3} = (-1)^{m_3}\,\conj{a_{\ell_3}^{-m_3}}$ shows
\[
  \beta_{\ell_1,\ell_2,\ell}(f)
  = (-1)^{\ell_1+\ell_2}\sqrt{2\ell+1}\,
    T_{\ell_1,\ell_2,\ell}(f),
\]
so $\beta$ and $T$ have the same vanishing locus. The 3j symbol is
invariant under even column permutations and picks up
$(-1)^{\ell_1+\ell_2+\ell_3}$ under odd ones. If two columns $i,j$
satisfy $\ell_i=\ell_j$, the inserted vectors $a_{\ell_i},a_{\ell_j}$
are literally the same coefficient vector of $f$, so swapping those
two columns leaves the contracted sum $T$ unchanged while multiplying
the 3j prefactor by $(-1)^{\ell_1+\ell_2+\ell_3}$. For odd parity this
forces $T=-T$, hence $T\equiv 0$ and $\beta\equiv 0$.
\end{proof}

\begin{remark}[Identically-zero candidates]
\label{rem:app-Tell-zero-rows}
Proposition~\ref{prop:app-odd-vanishing} rules out, for example,
$(\ell,\ell,\ell')$ with odd $\ell'$, $(2,3,2)$, $(4,5,4)$, and
$(6,7,6)$. Algorithm~\ref{alg:so3-augmented} discards every such
candidate, including the triples $(4,5,4)$ and $(6,7,6)$ of the lists
$\mathcal{T}^{\mathrm{small}}_\ell$ (Table~\ref{tab:Tsmall}).
\end{remark}

\begin{remark}[First nonzero odd-parity entries]\label{cor:app-first-odd}
The first odd-parity entries with all-distinct indices, and hence not
covered by Proposition~\ref{prop:app-odd-vanishing}, occur at
degree~$4$: $\beta_{2,3,4}$ and its orderings $\beta_{2,4,3}$ and
$\beta_{3,4,2}$. They are not identically zero (a random real signal
gives a nonzero value), so by Lemma~\ref{lem:app-parity} they change
sign under $T_R$.
\end{remark}

\begin{remark}[No odd-parity entries at $\ell \le 3$]\label{rem:app-no-parity-breaking-le3}
For degrees $\le 3$, every odd-parity triple $(\ell_1,\ell_2,\ell)$
has at least two indices equal (since $\max(\ell_1,\ell_2,\ell) \le 3$
and $\ell_1+\ell_2+\ell$ odd forces at least one pair to be equal in this range).
By Proposition~\ref{prop:app-odd-vanishing}, all odd-parity entries vanish.
Therefore, together with Lemma~\ref{lem:app-P-invariant}, every entry of degree $\le 3$
takes the same value on $f$ and on its mirror image $T_R\cdot f$.
\end{remark}

\begin{remark}[Distinct equations per degree]
\label{rem:app-Tell-rank-deficit}
On real signals each bispectral entry is one real number
(Lemma~\ref{lem:app-parity}), and all orderings of a triple are
proportional: $\beta_{\ell_1,\ell_2,\ell}
=(-1)^{\ell_1+\ell_2+\ell}\beta_{\ell_2,\ell_1,\ell}$ by the CG
exchange symmetry, and
$\beta_{\ell_1,\ell_2,\ell}=(-1)^{\ell_1}
\sqrt{(2\ell+1)/(2\ell_2+1)}\,\beta_{\ell_1,\ell,\ell_2}$ by the
3j form in the proof of Proposition~\ref{prop:app-odd-vanishing}.
Only distinct unordered triples therefore count toward the $2\ell{+}1$
real unknowns of $\mathbf{F}_\ell$, and for $\ell\ge 8$
Algorithm~\ref{alg:so3-augmented} lists one ordering per triple. At
$\ell\in\{8,9\}$ only 16 such triples with $a,b<\ell$ are nonzero,
so the power entry and self-couplings make up the remaining entries. The lists
$\mathcal{T}^{\mathrm{small}}_\ell$ at $4\le\ell\le 7$ (Table~\ref{tab:Tsmall}) contain several
orderings of some triples and supply 4, 4, 8, and 8 distinct
bispectral equations from their 10, 10, 13, and 14 nonzero entries.
\end{remark}

\subsection{Comparison with existing results}
\label{sec:app-comparison}

\begin{table}[h]
  \caption{Comparison of bispectral invariants for signals on $S^2$ or
    finite groups. \emph{constr.}\ marks a constructive proof. \citet{bendory2025orbit}
    recover signals on $\R^3$ sampled on concentric spherical shells by solving linear
    systems one degree at a time (frequency marching), and three shells suffice.}
  \label{tab:app-comparison}
  \centering
  \small
  \setlength{\tabcolsep}{4pt}
  \begin{tabular}{@{}l l c c@{}}
    \toprule
    Result & Invariant & Size & Complete? \\
    \midrule
    \citet{kakarala2012bispectrum}
      & Full bispec.\ (on $\SO(3)$) & $O(L^3)$ & Yes (constr.) \\
    \citet{edidin2024orbit}
      & Full bispec.\ (band-lim.)
      & $O(L^3)$ & Yes (generic) \\
    \citet{bendory2025orbit}
      & Freq.\ marching
      & $O(L^2)$/shell & Yes (3~shells) \\
    \citet{mataigne2024selective}
      & Sel.\ $G$-bispec.\ (finite)
      & $O(|G|)$ & Yes (generic) \\
    \textbf{This work}
      & \textbf{\begin{tabular}[t]{@{}l@{}}Augmented selective\\ bispectrum\end{tabular}}
      & $\boldsymbol{\Theta(L^2)}$
      & \textbf{Empirical (Section~\ref{sec:so3-recon})} \\
    \bottomrule
  \end{tabular}
\end{table}

\subsection{Complexity in pixel space}
\label{sec:app-pixel-complexity}

A signal band-limited at degree~$L$ is sampled on an equiangular grid
of $P = \Theta(L^2)$ pixels (e.g., $34\times 68$ at $L=16$ in
Table~\ref{tab:benchmarks}), so $L=\Theta(\sqrt{P})$. The augmented
selective bispectrum has $\Theta(L^2)=\Theta(P)$ entries, linear in the
number of pixels, against $\Theta(L^3)=\Theta(P^{3/2})$ for the full
bispectrum. Each entry contracts up to $O(L^2)$ CG-weighted products
(Eq.~\eqref{eq:app-bispectrum}), so the forward contraction as
implemented costs $O(L^4)=O(P^2)$ for $\Phi_{\mathrm{sel}}$ and
$O(L^5)=O(P^{5/2})$ for the full bispectrum. Counting the terms of the
contraction at $L=64$ ($128\times 256$ grid): the full bispectrum has
$70{,}257$ entries and $2.67\times 10^{8}$ terms, and
$\Phi_{\mathrm{sel}}$ has $5{,}610$ entries and $1.08\times 10^{7}$
terms, a $25\times$ reduction. Between $L=32$ and $L=64$ the term
counts grow as $L^{4.86}$ (full) and $L^{3.71}$ ($\Phi_{\mathrm{sel}}$), measured as
log-log slopes.

\section{Extended Benchmark Results}
\label{app:bench}

The scaling plots in this appendix extend Table~\ref{tab:benchmarks} and were produced by
\texttt{benchmarks/benchmark.py} on a single NVIDIA H100 80\,GB GPU
(PyTorch 2.11, \texttt{bispectrum} 0.4.1), with autograd enabled as during
training. Section~\ref{app:portability} covers other devices.

\subsection{Coefficient-count scaling}

\begin{figure}[h]
  \centering
  \includegraphics[width=\linewidth]{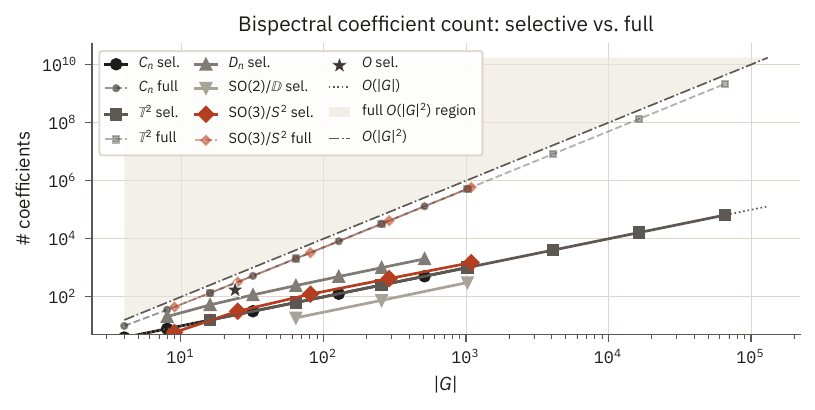}
  \caption{\textbf{The number of selective $G$-bispectral coefficients grows linearly with the
    group size. For full $G$-bispectra, it grows quadratically.} Invariant coefficients per
    signal against $|G|$: selective (solid) and full (dashed). For $\SO(2)$ on the disk and
    $\SO(3)$ on $S^2$, $|G|$ is the band-limited signal dimension. At $|G|{=}65{,}536$ the
    selective $\mathbb{T}^2$-bispectrum has $32{,}768\times$ fewer coefficients than the full
    $\mathbb{T}^2$-bispectrum ($6.6{\times}10^4$ vs.\ $2.1{\times}10^9$). The full bispectra
    for $D_n$, the disk $\mathbb{D}$ and $O$ are not implemented.}
  \label{fig:coeff-scaling}
\end{figure}

\subsection{Forward-pass timing}

\begin{figure}[h]
  \centering
  \includegraphics[width=\linewidth]{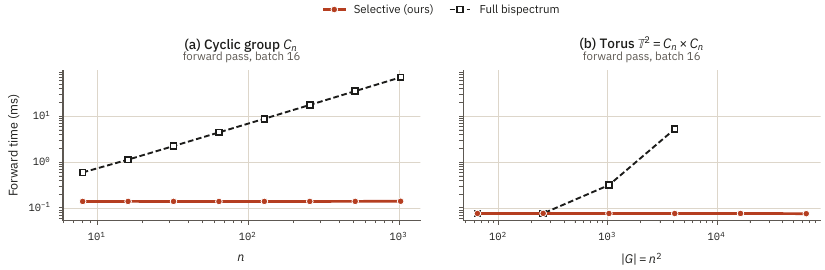}
  \caption{\textbf{The selective forward pass costs the same at every group size. The full
    bispectrum's cost grows with it.} Median forward time, H100, batch 16, autograd enabled.
    (a)~$C_n$: the selective pass takes 0.14\,ms at every $n$ from 8 to
    1{,}024. At $n{=}1{,}024$, the full bispectrum
    takes 70\,ms, $502\times$ longer.
    (b)~$\mathbb{T}^2$: 0.075--0.076\,ms up to $|G|{=}65{,}536$. The full
    bispectrum, run up to $|G|{=}4{,}096$ (8{,}390{,}656 coefficients),
    takes 5.2\,ms there, $69\times$ longer.}
  \label{fig:forward-timing}
\end{figure}

The selective forward pass is dominated by kernel-launch overhead
rather than arithmetic. The full bispectrum grows with $|G|$ and dominates the cost at large
group sizes.

\subsection{GPU throughput and inversion}

\begin{figure}[h]
  \centering
  \includegraphics[width=\linewidth]{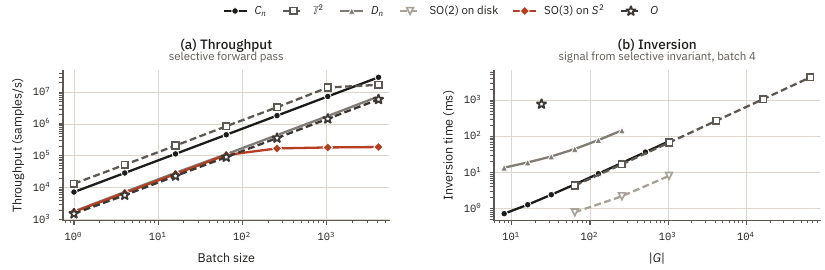}
  \caption{\textbf{Selective $G$-bispectra on finite groups reach millions of samples/s
    on one H100, and the library inverts every module except $\mathrm{SO}(3)$ on $S^2$.}
    (a)~Selective forward throughput against batch size, H100, autograd enabled ($C_{128}$,
    $C_{32}{\times}C_{32}$, $D_{32}$, $S^2$ at $L{=}16$, $O$): $C_n$ and $\mathbb{T}^2$ exceed $10^7$ samples/s,
    $O$ reaches $5.9{\times}10^{6}$, and $\mathrm{SO}(3)$ on $S^2$ saturates near
    $1.9{\times}10^{5}$. (b)~Time to recover a signal from its selective invariant (batch
    4). Log--log slopes against $|G|$ are 1.0 for $C_n$, 1.0 for $\mathbb{T}^2$, 0.7 for $D_n$ and 0.8 for the disk. $O$ (775\,ms) uses
    iterative Levenberg--Marquardt corrections. The $\mathrm{SO}(3)$ module on $S^2$ has no closed-form inversion.}
  \label{fig:gpu-scaling}
\end{figure}

Throughput grows linearly with batch size up to 4{,}096 for $C_n$, $D_n$,
and $O$ and up to 1{,}024 for $\mathbb{T}^2$: per-call cost stays nearly
constant, consistent with the launch-dominated forward pass above.
$\SO(3)$ on $S^2$ stops scaling at batch ${\approx}256$, where its
spherical harmonic transform and Clebsch--Gordan contractions, rather than launch
overhead, set the cost.

\subsection{Portability across devices and backends}
\label{app:portability}

Table~\ref{tab:portability} times the selective configurations of
Table~\ref{tab:benchmarks} on four GPUs and four host CPUs with a
single script (\texttt{experiments/tooling/benchmark\_devices.py}): batch\,=\,16, float32
input, \texttt{torch.inference\_mode()}, median over 30 repetitions of 20
forward passes after 10 warm-up passes, PyTorch 2.11, \texttt{bispectrum}
0.4.1 (0.4.0 for this table's H100 and A100 rows, a release that differs from 0.4.1 only on MPS).
Every module runs on every backend. All modules are sub-millisecond on the four
GPUs, which span data-center and consumer NVIDIA parts, two CUDA versions (12.8
and 13.0) and Apple MPS. On the consumer RTX 4090, \texttt{SO3onS2}, which
replays a CUDA graph, is faster than on the H100, while the other modules are
1.5--1.7$\times$ slower, a gap comparable to that between the two hosts'
single-thread CPU timings: at batch\,=\,16 these modules are bound by host-side
kernel launches rather than by the GPU. On CPU,
every module except \texttt{SO3onS2} is sub-millisecond, and \texttt{SO3onS2}
takes 2.2--45\,ms depending on host and thread count. CPU rows fix the number
of intra-op threads: with PyTorch's default of one thread per core, a shared
128-core host took 402\,ms for \texttt{SO3onS2}. MPS has no float64, so modules
are cast to float32 there (\texttt{bispectrum} $\geq$ 0.4.1). On CPU and CUDA, the
precomputed tables keep the dtype they are built with, and \texttt{SO2onDisk}
computes its disk harmonic transform in float64, so the MPS row is not
precision-matched with the others.

\begin{table}[h]
  \centering
  \caption{\textbf{Every module runs in under a millisecond on every device tested, except
    $\SO(3)$ on $S^2$ on CPU.} Selective forward pass (ms, median), batch\,=\,16,
    inference mode. Groups and sizes as in Table~\ref{tab:benchmarks} ($C_{128}$,
    $C_{32}{\times}C_{32}$, $D_{32}$, $O$, and the disk and $S^2$ at $L{=}16$), except that $S^2$ is
    sampled on the $64{\times}128$ Spherical MNIST grid instead of $34{\times}68$. On the
    H100, both grids take 0.20\,ms. This script and its timing loop differ from those of
    Table~\ref{tab:benchmarks}, so the H100 row differs from that table's Inf.\ column by up
    to 0.01\,ms.}
  \label{tab:portability}
  \small
  \setlength{\tabcolsep}{4pt}
  \begin{tabular*}{\linewidth}{@{\extracolsep{\fill}}lllrrrrrr@{}}
    \toprule
    Device & Backend & Threads & $C_n$ & $\mathbb{T}^2$ & $D_n$ & Disk & $S^2$ & $O$ \\
    \midrule
    H100 80\,GB & CUDA 12.8 & -- & 0.12 & 0.07 & 0.48 & 0.21 & 0.20 & 0.60 \\
    A100 80\,GB PCIe & CUDA 13.0 & -- & 0.16 & 0.10 & 0.64 & 0.31 & 0.26 & 0.82 \\
    RTX 4090 24\,GB & CUDA 12.8 & -- & 0.20 & 0.12 & 0.76 & 0.34 & 0.18 & 0.92 \\
    Apple M5 Pro & MPS (fp32) & -- & 0.08 & 0.07 & 0.30 & 0.69 & 0.82 & 0.41 \\
    \midrule
    Intel Xeon 8480+ & CPU & 1 & 0.05 & 0.18 & 0.29 & 0.16 & 27.8 & 0.25 \\
     & & 8 & 0.08 & 0.14 & 0.32 & 0.13 & 7.36 & 0.28 \\
    AMD EPYC 7763 & CPU & 1 & 0.06 & 0.22 & 0.37 & 0.16 & 32.2 & 0.35 \\
     & & 8 & 0.08 & 0.17 & 0.45 & 0.19 & 5.46 & 0.45 \\
    AMD EPYC 7K62 & CPU & 1 & 0.08 & 0.26 & 0.44 & 0.18 & 45.4 & 0.44 \\
     & & 8 & 0.10 & 0.30 & 0.51 & 0.23 & 13.9 & 0.55 \\
    Apple M5 Pro & CPU & 1 & 0.02 & 0.10 & 0.10 & 0.04 & 5.75 & 0.11 \\
     & & 8 & 0.02 & 0.15 & 0.14 & 0.07 & 2.22 & 0.12 \\
    \bottomrule
  \end{tabular*}
\end{table}

\section{$G$-Bispectral Pooling: Extended Results}
\label{app:pooling}

\subsection{Invariant maps in $G$-CNNs (Figure~\ref{fig:pooling-grid})}
\label{app:pooling-grid}

\paragraph{Protocol.}
The rotated test averages the metric over the test set transformed by every element of $G$,
identity included. On OrganMNIST3D the 24 octahedral rotations permute voxels exactly, so
$G$-invariant models score identically on the canonical and rotated test sets. On PCam, the
$45^\circ$ rotations are bilinearly interpolated, so invariance is exact only at multiples of
$90^\circ$. Aug.\ CNN is a non-equivariant CNN (3D for OrganMNIST3D) trained with augmentation
over $G$, so the rotated test is in-distribution for it. All other models are $G$-equivariant,
trained on canonical data only, and end in an invariant map over the group dimension: a maximum (Max
pool, and the NormReLU, Gated, and Fourier-ELU models, named after their backbone nonlinearity), the $\ell_2$ norm (Norm
pool, an incomplete invariant), or the selective $G$-bispectrum (complete). Max pool and
Bispectrum on OrganMNIST3D, and Norm pool and Bispectrum on PCam, share a ReLU backbone and differ
only in the invariant map. OrganMNIST3D's Norm pool uses NormReLU.

\paragraph{Held fixed.}
Left: full training set, at the smallest width (OrganMNIST3D) or ${\sim}100$K-parameter
configurations (PCam). Middle: full training set (OrganMNIST3D) or $n = 12.5$K (PCam). Right:
smallest width (OrganMNIST3D) or ${\sim}100$K parameters (PCam). The label ``all'' is the full training set,
and the PCam axis is broken before it (262K examples). The panels of a row share the $y$-axis.

\paragraph{Statistics.}
Mean $\pm$ sample standard deviation over 10 seeds (OrganMNIST3D) and 3 seeds (PCam, 5 at
$n \le 2.5$K). $p$-values are two-sided Welch $t$-tests, uncorrected for multiplicity. The 46
comparisons are 25 on PCam (5 training-set sizes, 5 baselines) and 21 on OrganMNIST3D (4
training subsets at the smallest width and the full training set at 3 widths, 3 baselines).
The PCam width sweep is not tested, because the methods' parameter counts differ at each
growth rate. On
OrganMNIST3D, bispectral pooling beats max pooling by 5.0\,pp ($p < 10^{-3}$). Its smallest
model (463K parameters) beats max pooling's largest (6.0M) by 2.6\,pp
($p < 10^{-3}$), and at $n = 500$ it beats max pooling at $n = 971$ by
3.6\,pp ($p = 0.0017$). On PCam, no model is significantly better than bispectral
pooling at any training-set size. The closest is Gated on the full training set (${+}0.016$ AUC,
$p = 0.093$). With 3 seeds per arm and PCam's seed spread, these tests have 80\% power only for
gaps of about 0.025 AUC, so ``on par'' means no difference at that resolution.

\subsection{Model variants}

\paragraph{PCam.}
All equivariant models share a $C_8$-equivariant
DenseNet~\citep{huang2017densenet} backbone with block configuration
$(4, 4, 4)$, 24 initial channels, and compression factor 0.5,
following the architecture of \citet{veeling2018rotation} but
replacing the final invariant map. Six model variants are evaluated:
\begin{itemize}
  \item \textbf{Aug.\ CNN}: non-equivariant DenseNet trained with augmentation over $C_8$
    and global average pooling. No group dimension, so ${\sim}8\times$ fewer parameters at
    matched growth rate.
  \item \textbf{NormReLU}, \textbf{Gated}, \textbf{Fourier-ELU}: $C_8$-equivariant
    backbones with norm-based, gated (doubling internal channels) or Fourier-space ELU
    nonlinearities, each followed by max pooling over the group dimension.
  \item \textbf{Norm pool}: ReLU backbone, $\ell_2$ norm over the group dimension.
  \item \textbf{Bispectrum}: ReLU backbone, selective $G$-bispectrum via
    \texttt{CnonCn}$(n{=}8)$ as the terminal invariant map.
\end{itemize}
The data-efficiency and full-data panels use growth rates matched to ${\sim}100$K
parameters. The parameter-efficiency panel sweeps the DenseNet growth rate
$k \in \{3, 4, 6, 8, 12\}$ for equivariant models and $k \in \{6, 12, 20, 30, 35\}$ for
Aug.\ CNN at $n = 12.5$K, from ${\sim}30$K to ${\sim}1.6$M parameters.

\paragraph{OrganMNIST3D.}
All equivariant models share an $O$-equivariant 3D ResNet whose kernel rotations are exact
voxel permutations on the $3^3$ grid. The channel widths $(4, 8)$, $(8, 16)$, and $(16, 32)$ give 463K, 1.7M, and 6.3M parameters with the bispectral head (172 \texttt{OctaonOcta}
coefficients per channel) and 374K, 1.5M, and 6.0M with max pooling. Aug.\ CNN is a
non-equivariant 3D CNN (16K, 63K, and 251K parameters) trained with one of the 24 octahedral rotations per sample. Norm pool
uses a NormReLU backbone. Training-set sizes are $n \in \{50, 100, 250, 500, 971\}$ at
width $(4, 8)$.

\subsection{Pooling head}
\label{app:pooling-head}

The bispectral head computes $\mathrm{ReLU}(\mathrm{proj}(\mathrm{signlog}(\beta)))$,
with $\mathrm{signlog}$ the signed logarithm of the real and imaginary parts and $\mathrm{proj}$
the learned $1{\times}1$ projection. Placing the ReLU first,
$\mathrm{proj}(\mathrm{ReLU}(\mathrm{signlog}(\beta)))$, zeroes every negative $G$-bispectral
coefficient before the projection. Table~\ref{tab:head-ablation} compares the two heads and two
variants that append the power spectrum. The head was chosen on validation AUC alone.
Spherical MNIST does not use this head.
On $O$ the truncated invariant keeps full Jacobian rank
(200 of 200 random inputs) yet costs ${\approx}5$\,pp, while on $C_8$ it is rank-deficient
(24 of 200) yet costs nothing measurable, so Jacobian rank does not explain the gap. We suspect that the
surviving subset changes with every input, so the fixed projection reads different coordinates
at every voxel, and the zeroed entries pass no gradient. The rule we adopt is to apply only fixed
injective maps (real--imaginary split, signed log) before the first learned layer. Because we compared four heads for the bispectrum but only one per baseline,
Table~\ref{tab:head-ablation} also reports the ReLU-first head. Norm pooling has no truncated variant, because its outputs are
non-negative and a ReLU leaves them unchanged.

\begin{table}[h]
  \centering
  \caption{\textbf{Moving the ReLU after the projection adds 4--5\,pp on
    OrganMNIST3D at full data.} Test accuracy (OrganMNIST3D, \%) or AUC (PCam) at the
    best-validation checkpoint, mean $\pm$ std over 3 seeds. The ReLU-first head applies the
    ReLU to the signed-log bispectrum before the $1{\times}1$ projection, which zeroes every
    negative coefficient. The head we use (bold) was chosen on validation AUC alone, by
    mean rank over the five settings. It makes no significant difference at $n{=}100$ or on
    PCam. $p$: Welch test against the ReLU-first head.}
  \label{tab:head-ablation}
  \small
  \setlength{\tabcolsep}{3pt}
  \begin{tabular*}{\linewidth}{@{\extracolsep{\fill}}lccccc@{}}
    \toprule
    & \multicolumn{3}{c}{OrganMNIST3D (\%)} & \multicolumn{2}{c}{PCam (AUC)} \\
    \cmidrule(lr){2-4} \cmidrule(lr){5-6}
    Head & 6.3M & 463K & 463K, $n{=}100$ & $n{=}100$ & $n{=}12.5$K \\
    \midrule
    ReLU, then projection & $73.9 {\pm} 1.5$ & $73.6 {\pm} 2.0$ & $61.0 {\pm} 0.9$ & $.812 {\pm} .024$ & $.939 {\pm} .011$ \\
    \textbf{Projection, then ReLU} & $\mathbf{77.9 {\pm} 1.9}$ & $\mathbf{78.6 {\pm} 0.3}$ & $\mathbf{61.1 {\pm} 2.6}$ & $\mathbf{.824 {\pm} .021}$ & $\mathbf{.936 {\pm} .007}$ \\
    \quad $p$ vs.\ ReLU first & $0.045$ & $0.049$ & $0.93$ & $0.55$ & $0.73$ \\
    \quad + power spectrum & $78.1 {\pm} 1.2$ & $76.9 {\pm} 2.4$ & $62.4 {\pm} 2.7$ & $.813 {\pm} .031$ & $.940 {\pm} .008$ \\
    \quad + power spectrum, cube root & $77.5 {\pm} 0.7$ & $79.0 {\pm} 3.7$ & $62.7 {\pm} 2.8$ & $.821 {\pm} .003$ & $.936 {\pm} .005$ \\
    \midrule
    Max pool & $76.0 {\pm} 2.1$ & $73.1 {\pm} 0.6$ & $63.1 {\pm} 1.8$ & -- & -- \\
    Norm pool & $68.2 {\pm} 3.7$ & $63.3 {\pm} 2.7$ & $57.9 {\pm} 3.0$ & $.802 {\pm} .038$ & $.933 {\pm} .011$ \\
    Aug.\ CNN & -- & -- & -- & $.813 {\pm} .042$ & $.928 {\pm} .009$ \\
    \bottomrule
  \end{tabular*}
\end{table}

\subsection{Published baselines}

Table~\ref{tab:pcam-context} places our results in the context of
published methods on PCam. Our models are trained under one fixed
protocol for every method rather than tuned per method. At full data
and ${\sim}100$K parameters, bispectral pooling reaches a test AUC of 0.933.

\begin{table}[h]
  \caption{Published baselines on PatchCamelyon from
    \citet{veeling2018rotation}, trained on 100\% data.}
  \label{tab:pcam-context}
  \centering
  \small
  \begin{tabular}{llrcc}
    \toprule
    Method & Reference & Params & ACC & AUC \\
    \midrule
    P4M-DenseNet ($D_4$)  & \citet{veeling2018rotation} & 119K & 0.898 & 0.963 \\
    P4M-DenseNet M        & \citet{veeling2018rotation} &  19K & 0.893 & 0.958 \\
    DenseNet (augmented)   & \citet{veeling2018rotation} & 128K & 0.881 & 0.951 \\
    DenseNet (no aug.)     & \citet{veeling2018rotation} & 128K & 0.876 & 0.955 \\
    \bottomrule
  \end{tabular}
\end{table}

Table~\ref{tab:organ3d-context} situates our OrganMNIST3D results against published
methods on the same benchmark. Our equivariant models (0.37M--6.3M parameters) share one
plain backbone to isolate the pooling mechanism, and bispectral pooling reaches $78.5\%$ at
0.46M parameters. Equivariant models of similar or smaller size (ILPOResNet-50, EquiLoPO)
reach 86.6--87.9\%, so our absolute accuracy is not competitive. The comparison our experiments
support is between pooling layers on a shared backbone.

\begin{table}[h]
  \caption{Published baselines on OrganMNIST3D for context.}
  \label{tab:organ3d-context}
  \centering
  \small
  \begin{tabular}{llrcc}
    \toprule
    Method & Reference & Params & ACC & AUC \\
    \midrule
    ResNet-18 + 3D           & \citet{medmnistv2} & 33M  & 0.907 & 0.996 \\
    ResNet-18 + ACS          & \citet{medmnistv2} & 11M  & 0.900 & 0.994 \\
    ILPOResNet-50            & \citet{zhemchuzhnikov2024ilponet}       & 38K  & 0.879 & 0.992 \\
    EquiLoPO (local train.)  & \citet{zhemchuzhnikov2025equilopo}      & 418K & 0.866 & 0.991 \\
    SE3MovFrNet              & \citet{sangalli2023movfrnet}     & ---  & 0.745 & ---   \\
    Regular SE(3) conv       & \citet{kuipers2023regular}   & 172K & 0.698 & ---   \\
    \bottomrule
  \end{tabular}
\end{table}

\section{Experimental Details}
\label{app:exp}

\subsection{Training protocol}

Every run trains for a fixed number of epochs that ends a cosine cycle of
SGDR~\citep{loshchilov2017sgdr} ($T_0 = 10$ epochs, $T_{\mathrm{mult}} = 2$, so cycles end at
epochs 10, 30, 70, \dots) and is evaluated at its best-validation checkpoint. There is no early
stopping. An epoch draws at least 10{,}000 training examples, cycling small training sets, so
low-data runs take as many optimizer steps per epoch as larger ones. Table~\ref{tab:hyperparams}
lists the remaining settings.

\begin{table}[h]
  \caption{Training hyperparameters. All experiments use AdamW~\citep{loshchilov2019adamw}
    with weight decay $10^{-4}$ and the SGDR schedule above.}
  \label{tab:hyperparams}
  \centering
  \small
  \setlength{\tabcolsep}{4pt}
  \begin{tabular}{lccc}
    \toprule
    & PCam & OrganMNIST3D & Sph.\ MNIST \\
    \midrule
    Learning rate   & $10^{-3}$ & $10^{-3}$ & tuned$^\dagger$ \\
    Batch size      & 64 & 16 & tuned$^\dagger$ \\
    Epochs          & 30 (10 at full data) & 70 & 30 \\
    Checkpoint      & best val.\ AUC & best val.\ accuracy & best val.\ accuracy \\
    Min.\ examples per epoch & 10{,}000 & 10{,}000 & 10{,}000 \\
    Seeds           & 3 (5 at $n \le 2.5$K) & 10 & 5 \\
    Mixed precision & Yes & No & No \\
    \bottomrule
  \end{tabular}

  \vspace{0.3em}
  {\small $^\dagger$Selected on full-data validation accuracy and reused at every $n$.
    Bispectrum uses batch 32, lr $10^{-3}$ (from batch $\{32, 256\}\times$ lr
    $\{1, 3, 5\}\times 10^{-3}$). S2CNN+BN uses batch 32, lr $5\times 10^{-3}$ (from batch
    $\{32, 256\}\times$ lr $\{1, 5\}\times 10^{-3}$). The tensor product uses 24 channels,
    $\ell_{\max} = 6$, lr $3\times 10^{-3}$ (from 8 configurations of channels, $\ell_{\max}$
    and lr). The CNN and power spectrum use batch 256, lr $10^{-3}$.}
\end{table}

\subsection{Hardware}

Training ran on three machines: $8\times$ A100 (40\,GB), $2\times$ A100 (80\,GB), and one
H100 (80\,GB). Spherical MNIST latencies in Figure~\ref{fig:smnist} are measured on an
idle H100.

\subsection{Data preprocessing}

\paragraph{PCam.}
PatchCamelyon~\citep{veeling2018rotation} consists of 327{,}680 RGB
patches of size $96 \times 96$ pixels extracted from whole-slide
images of lymph node sections~\citep{bejnordi2017camelyon}. We use
the standard train/val/test split. Images are loaded from HDF5 files
and normalized to $[0, 1]$. Aug.\ CNN is trained with augmentation over $C_8$. All
equivariant models are trained without augmentation. Training-set sizes
$n \in \{100, 500, 2.5\mathrm{K}, 12.5\mathrm{K}\}$ are uniform random subsets of the 262K
training patches, drawn per seed.

\paragraph{OrganMNIST3D.}
OrganMNIST3D~\citep{medmnistv2} contains 1{,}742 abdominal CT organ
volumes of size $28 \times 28 \times 28$ voxels with 11 organ classes
(train/val/test = 971/161/610, as distributed by \texttt{medmnist} 3.0.2). Volumes are normalized to $[0, 1]$.
Aug.\ CNN is trained with one of the 24 octahedral rotations per sample.
Equivariant models are trained without augmentation. Kernel rotations are implemented as exact
voxel permutations on the $3^3$ grid, avoiding interpolation artifacts. Subsets with
$n < 971$ are uniform random subsets, drawn per seed.

\paragraph{Spherical MNIST.}
MNIST digits are projected onto the sphere via stereographic
projection following \citet{cohen2018spherical} and sampled on a
$64 \times 128$ equiangular grid. The augmented selective spherical bispectrum is
computed at $L{=}15$, producing 380 invariants, stored as complex numbers. The model input is obtained by applying
$\operatorname{sign}(x)\log(1+|x|)$ separately to the real and
imaginary parts of each coefficient, giving 760 real features. These features are
fed to a 3-layer MLP: Linear(760, 256)--BN--ReLU--Linear(256,
128)--BN--ReLU--Linear(128, 10). The power spectrum MLP uses the
same structure with hidden width 128: Linear(16, 128)--BN--ReLU--Linear(128,
64)--BN--ReLU--Linear(64, 10). The standard CNN has 4 convolutional
layers (16--32--64--64 channels) followed by adaptive average pooling
and a Linear(1024, 128)--ReLU--Linear(128, 10) classifier. S2CNN+BN is the spherical CNN of
\citet{cohen2018spherical} (features 20 and 40, bandwidths 10 and 6) with batch normalization
after each convolution, 155K parameters. The tensor-product network is an $\SO(3)$-equivariant
Clebsch--Gordan network with 24 channels per degree up to $\ell_{\max} = 6$, 386K parameters.
Latency is the median end-to-end forward time, from the spherical grid to the logits, in
float32 at batch 16 on an idle H100. With the published settings (no BatchNorm), S2CNN reaches
80.7\% NR/R in our pipeline, against 94\% reported by \citet{cohen2018spherical}.

\subsection{Statistics}

All results report the mean $\pm$ sample standard deviation over seeds, evaluated at the
best-validation checkpoint. Comparisons are two-sided Welch $t$-tests. The pooling comparisons
(Figure~\ref{fig:pooling-grid}) are reported uncorrected, and the 10 Spherical MNIST
data-efficiency comparisons (Figure~\ref{fig:smnist}) are Holm-corrected jointly. The rotated test for $C_8$ and $O$ averages over every group element. For $\SO(3)$, it
applies one random rotation per test digit.

\subsection{Spherical MNIST: capacity of the power spectrum}
\label{app:smnist-capacity}

The power spectrum is a weak baseline at any capacity. Table~\ref{tab:smnist-capacity} sweeps
the MLP width on both invariants at the default training settings (batch 256, lr $10^{-3}$,
30 epochs): the power spectrum saturates near 80\% at 550K parameters, below the bispectrum at
25K.

\begin{table}[h]
  \caption{Rotated-test accuracy on Spherical MNIST (NR/R, \%) against MLP size, mean $\pm$ std
    over 5 seeds at the default training settings.}
  \label{tab:smnist-capacity}
  \centering
  \small
  \begin{tabular}{rcrc}
    \toprule
    \multicolumn{2}{c}{Power spectrum (16 features)} & \multicolumn{2}{c}{Bispectrum (760 features)} \\
    \cmidrule(lr){1-2} \cmidrule(lr){3-4}
    Params & NR/R & Params & NR/R \\
    \midrule
    3.7K & $75.79 \pm 0.19$ & 25K  & $94.05 \pm 0.23$ \\
    11K  & $77.72 \pm 0.17$ & 51K  & $94.90 \pm 0.29$ \\
    39K  & $78.93 \pm 0.17$ & 107K & $95.29 \pm 0.14$ \\
    144K & $79.51 \pm 0.17$ & 230K & $95.52 \pm 0.08$ \\
    550K & $80.08 \pm 0.20$ & 525K & $95.74 \pm 0.13$ \\
    \bottomrule
  \end{tabular}
\end{table}

\subsection{Spherical MNIST: ablation of the augmented entries}
\label{app:smnist-ablation}

Table~\ref{tab:smnist-augmentation} builds the augmented selective spherical bispectrum up
entry type by entry type. Only the Clebsch--Gordan power spectra measurably help
accuracy. The self-couplings do not change accuracy and are included for the Jacobian rank.
The classifier uses all 760 real features.

\begin{table}[h]
  \centering
  \caption{\textbf{Of the augmentation steps, only the Clebsch--Gordan power spectra measurably
    help.} Spherical MNIST accuracy (\%) as the augmented selective spherical bispectrum at $L{=}15$ is
    built up. Each row adds entries to the one above, with the MLP width fixed and the other
    entries masked to zero. Mean $\pm$ std over 5 seeds, at the default training settings
    (batch 256, learning rate $10^{-3}$, 30 epochs), whereas Figure~\ref{fig:smnist} uses tuned
    settings. $p$: Welch test of NR/R against the row above.}
  \label{tab:smnist-augmentation}
  \small
  \begin{tabular*}{\linewidth}{@{\extracolsep{\fill}}lrccc@{}}
    \toprule
    Entries & Real features & NR/NR & NR/R & $p$ \\
    \midrule
    Base bispectral entries & 480 & $95.20 {\pm} 0.09$ & $95.18 {\pm} 0.11$ & -- \\
    + even self-couplings $\beta_{\ell,\ell,\ell'}$ & 584 & $95.23 {\pm} 0.10$ & $95.19 {\pm} 0.08$ & $0.87$ \\
    + Clebsch--Gordan power spectra & 738 & $95.50 {\pm} 0.05$ & $95.51 {\pm} 0.05$ & $<10^{-3}$ \\
    + degree power, other self-couplings (full) & 760 & $95.58 {\pm} 0.10$ & $95.58 {\pm} 0.08$ & $0.14$ \\
    \bottomrule
  \end{tabular*}
\end{table}

\subsection{Recovering signals from the augmented selective spherical bispectrum
  (Figure~\ref{fig:so3-reconstruction})}
\label{app:so3-recon}

\paragraph{Protocol.} Signals are band-limited to $L{=}15$ and sampled on a
$128{\times}256$ grid. The augmented selective spherical bispectrum has 380 coefficients. Targets are 8 Spherical
MNIST digits, each in its original and one random orientation, and 16 random band-limited
signals. Gradient descent on the image starts from noise (4 restarts, 8{,}000 Adam steps,
projected onto the band limit). The degree-by-degree solver works on the 256 real
spherical-harmonic coefficients in double precision: Levenberg--Marquardt with the exact
Jacobian fits degrees 0--4 from random starts, then adds one degree at a time and refits all
degrees so far, keeping up to 12 candidates. Each result is rotated onto the target before
measuring the relative image error. The solver lives in
\texttt{experiments/spherical\_mnist\_reconstruction/degree\_marching.py}.

\begin{table}[h]
  \caption{Reconstruction from the augmented selective spherical bispectrum at $L{=}15$, by method and
    target set.}
  \label{tab:so3-recon}
  \centering
  \small
  \begin{tabular}{@{}lcc@{}}
    \toprule
    & Spherical MNIST & Random signals \\
    \midrule
    \multicolumn{3}{@{}l}{\emph{Gradient descent from noise}} \\
    \quad invariant matched to $<1\%$ & 14 / 16 & 0 / 16 \\
    \quad image within 0.3 after alignment & 9 / 16 & 0 / 16 \\
    \midrule
    \multicolumn{3}{@{}l}{\emph{Degree-by-degree solver (Figure~\ref{fig:so3-reconstruction})}} \\
    \quad full bispectrum matched to $10^{-12}$ & 16 / 16 & 16 / 16 \\
    \quad image error after alignment & 0.0014--0.0035 & 0.0056--0.0070 \\
    \midrule
    Different-orbit signals with equal invariants found & 0 & 0 \\
    \bottomrule
  \end{tabular}
\end{table}

\paragraph{Results.} Gradient descent recovers
no random signal: it stalls at a 7.5\% median invariant error. The degree-by-degree solver
recovers all 32 targets up to the interpolation error of the alignment. Refitting all degrees after each new one is necessary: holding the lower degrees fixed instead lets the residual drift just past the
$10^{-10}$ acceptance threshold, which lost a ``1'' at $\ell{=}15$ and a ``0'' at $\ell{=}10$.
The remaining image error comes from the alignment, which rotates the reconstruction by
interpolating on the grid. As a check, we rotated each target exactly, through its
spherical-harmonic expansion, and aligned it back in the same way. This leaves
$0.0014$--$0.0034$ (digits) and $0.0060$--$0.0076$ (random signals), the same range as the
reconstructions. No run found two signals in different orbits with equal invariants,
consistent with conjectured generic completeness.

\subsection{Asset licenses}

\begin{itemize}[nosep]
  \item \textbf{PCam}: CC0 1.0 Universal (public domain).
  \item \textbf{OrganMNIST3D}: CC BY 4.0, loaded with \texttt{medmnist} 3.0.2 (Apache-2.0).
  \item \textbf{Spherical MNIST}: generated from MNIST, which is released under CC BY-SA 3.0.
  \item \textbf{Code}: PyTorch 2.11 (BSD-3-Clause), NumPy (BSD-3-Clause),
    \texttt{torch\_harmonics} 0.9.2 (BSD-3-Clause), and \texttt{e3nn} 0.6.0 (MIT).
\end{itemize}

\section{API Example}
\label{app:api-example}

\begin{lstlisting}
import torch
from bispectrum import CnonCn, OctaonOcta, SO2onDisk

f = torch.randn(16, 8)  # 16 signals on C_8
bsp = CnonCn(n=8)  # selective by default
beta = bsp(f)  # (16, 8) invariant coefficients
f_rec = bsp.invert(beta)  # signal recovered from beta

# The same interface for other group actions
beta_o = OctaonOcta()(torch.randn(16, 24))  # signals on O
beta_d = SO2onDisk(L=32)(torch.randn(16, 32, 32))  # images on the disk
\end{lstlisting}

\end{document}